\documentclass[letterpaper]{article} 
\usepackage{aaai2026}  
\usepackage{times}  
\usepackage{helvet}  
\usepackage{courier}  
\usepackage[hyphens]{url}  
\usepackage{graphicx} 
\usepackage{natbib}  
\usepackage{caption} 
\usepackage{algorithm}
\usepackage{algorithmic}

\usepackage{newfloat}
\usepackage{listings}
\DeclareCaptionStyle{ruled}{labelfont=normalfont,labelsep=colon,strut=off} 
\floatstyle{ruled}
\newfloat{listing}{tb}{lst}{}
\floatname{listing}{Listing}
\usepackage{amsfonts}
\usepackage{amsmath}
\usepackage{amsthm}
\usepackage{mathtools}
\usepackage{cleveref}
\usepackage{subfigure}
\usepackage{enumitem}
\usepackage{booktabs}

\theoremstyle{plain}
\newtheorem{theorem}{Theorem}
\newtheorem{lemma}[theorem]{Lemma}
\newtheorem{proposition}[theorem]{Proposition}
\newtheorem{corollary}[theorem]{Corollary}
\theoremstyle{definition}
\newtheorem{assumption}{Assumption}

\DeclareMathOperator{\argmin}{arg\,min}

\title{Align When They Want, Complement When They Need!\\ Human-Centered Ensembles for Adaptive Human-AI Collaboration}
\author{
    Hasan Amin,
    Ming Yin,
    Rajiv Khanna
}
\affiliations{
    Purdue University, USA\\
    \{hasanamin, mingyin, rajivak\}@purdue.edu
}

\begin{document}

\maketitle

\begin{abstract}
In human-AI decision making, designing AI that complements human expertise has been a natural strategy to enhance human-AI collaboration, yet it often comes at the cost of decreased AI performance in areas of human strengths. This can inadvertently erode human trust and cause them to ignore AI advice precisely when it is most needed.
Conversely, an aligned AI fosters trust yet risks reinforcing suboptimal human behavior and lowering human-AI team performance.
In this paper, we start by identifying this fundamental tension between performance-boosting (i.e., {\em complementarity}) and trust-building (i.e., {\em alignment}) as an inherent limitation of the traditional approach for training a single AI model to assist human decision making.
To overcome this, we introduce a novel,
{\em human-centered adaptive AI ensemble} that strategically toggles between two specialist AI models---the aligned model and the complementary model---based on contextual cues, using an elegantly simple yet provably near-optimal {\em Rational Routing Shortcut} mechanism.
Comprehensive theoretical analyses elucidate
why the adaptive AI ensemble is effective and when it yields maximum benefits.
Moreover, experiments on both simulated and real-world data show that when humans are assisted by the adaptive AI ensemble in decision making, they can achieve significantly higher performance than when they are assisted by single AI models that are trained to either optimize for their independent performance or even the human-AI team performance.
\end{abstract}

\begin{links}
    \link{Supplementary material}{https://github.com/shasanamin/aaai26-adaptive-ai}
\end{links}

\begin{figure}[th]
    \centering
    \includegraphics[width=\linewidth]{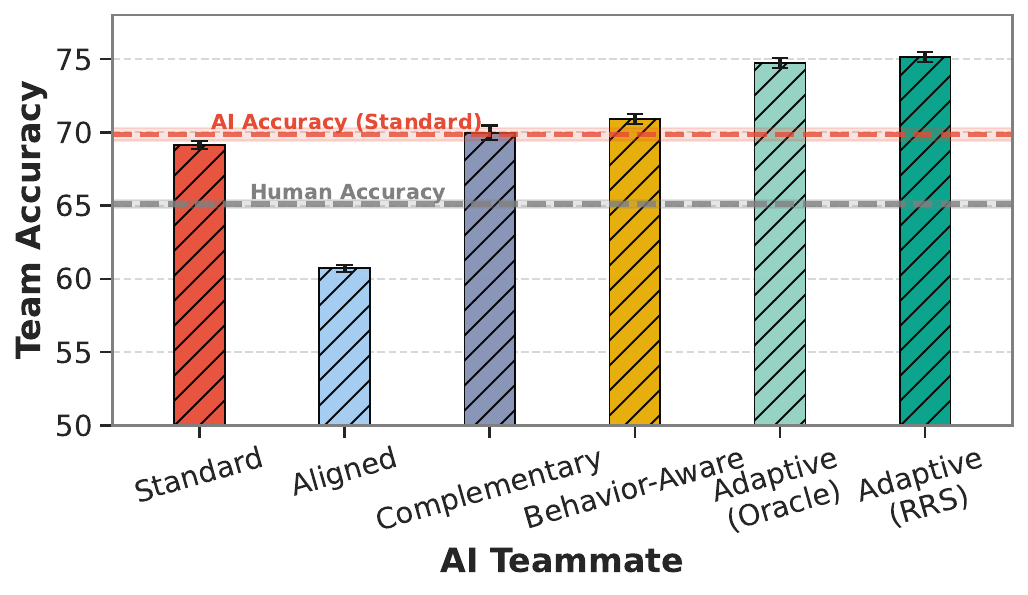}
    \caption{Summary comparison of human-AI team performance across AI design paradigms. Standard AI optimizes independent accuracy; aligned and complementary AIs specialize in trust-building and error-correction, respectively; behavior-aware AI optimizes team loss under our human interaction model.
    Our proposed adaptive AI ensemble, including the more practical one using Rational Routing Shortcut (RRS), strategically toggles between aligned and complementary AI, and achieves the highest team accuracy. Results shown here correspond to the behavior-grounded evaluation on \textsf{WoofNette} data.
    }
    \label{fig:HAI-performance}
\end{figure}

\section{Introduction}
AI systems are increasingly being developed to support human decision making in various domains, yet the grand promise of effective human-AI collaboration remains largely unfulfilled. Contrary to intuition, human-AI teams routinely underperform compared to their best constituent, i.e., the human or AI alone
\cite{dell2023super,agarwal2023combining,vaccaro2024combinations,goh2024large}. This recurring failure signals a fundamental gap in our current understanding and raises a deeper question at the core of collaborative AI design: \textit{what makes an AI a truly effective teammate?}

A dominant paradigm in human-AI collaboration emphasizes \textit{complementarity}: optimizing AI to be correct on instances where human judgment is weak \cite{madras2018predict,lai2022human,steyvers2022bayesian,zhang2022you,holstein2021designing}. While this approach can elevate team performance in principle, it often neglects a crucial psychological and behavioral reality: human trust in AI is not merely a function of AI's objective accuracy, but hinges disproportionately on AI's alignment with human judgment, especially on instances where humans feel more confident \cite{lu2021human,chong2022human}. When such complementary AI diverges from human intuition on instances where humans are confident, even if wrongly so, it can degrade trust and reliance, resulting in users paradoxically ignoring AI advice precisely when it would have been most valuable~\cite{lu2021human,lu2024does}.
Conversely, an AI that focuses on \textit{alignment} with human judgment may foster trust but also reinforce suboptimal human decisions, squandering the performance gains that complementary AI can deliver \cite{grgic2022taking, askell2021general, ouyang2022training, guo2024controllable}.
Recognizing the importance of human-AI team performance, some recent work has attempted to optimize for team decision making directly, taking human-AI interaction behavior into account \cite{bansal2021most,mahmood2024designing}. Such behavior-aware AI has been a promising proof-of-concept, already leading to some gains in team performance compared to human or AI alone, despite strong assumptions like rationality \cite{bansal2021most} or confidence thresholding \cite{mahmood2024designing}.

In this work, we start by extending the behavior-aware AI framework to incorporate more realistic human behavior, capturing the dual need to foster trust through alignment and boost performance through complementarity.
We introduce a behaviorally grounded \textit{Confidence-Gated Probabilistic Reliance (CGPR)} model that
formally links human self-confidence and trust to the likelihood of accepting AI recommendations.
Building on this, we derive the optimal single-model, behavior-aware AI that directly optimizes team decision making performance.
In doing so, we uncover an intriguing phenomenon: the ``\textit{complementarity-alignment tradeoff}.''
That is, even the best single-model, behavior-aware AI cannot simultaneously optimize for both trust (through alignment) and team performance (through complementarity) across all regions of the decision space.

In light of these fundamental limitations of the single-model paradigms,
we supplement the behavior-aware AI with a new direction: \emph{human-centered adaptive AI ensembles} that dynamically toggle between alignment and complementarity---preserving trust where it matters most, and providing support where it helps most.
In particular, we propose a practicable instantiation of the adaptive AI framework that learns separate  specialist AI models for alignment and complementarity, and then dynamically selects between them at inference time using \textit{Rational Routing Shortcut (RRS)}, a simple and theoretically principled mechanism that approximates human-aware, optimal routing without requiring access to private or hard-to-measure human states.
We further provide rigorous theoretical analyses quantifying when and why adaptive AI yields substantial team performance gains over any single-model alternative.
Finally, we showcase the generalizability and robustness of our theoretical insights through simulations and a behavior-grounded image decision making benchmark constructed from real human accuracy and confidence data.
Notably, our adaptive AI ensemble achieves remarkable empirical improvements in enhancing human-AI team accuracy in decision making in practice: up to $9\%$ increase in team accuracy over standard AI optimized solely for independent accuracy, and $6\%$ over behavior-aware AI optimized explicitly for team performance (see Figure~\ref{fig:HAI-performance})---remarkably, while relying on specialized AI components individually less accurate than the standard AI model.

While empirical work on human-AI collaboration is rapidly growing, a formal understanding of the interplay between team performance, complementarity and alignment is severely lacking.
Our work provides one of the first foundational frameworks to characterize and address this gap, and highlights a shift in human-AI collaboration: moving from static, single-model approaches to \emph{adaptive, behaviorally grounded, and theoretically grounded} systems that aim to ``align when desired, and complement when needed.''
Our main contributions are:
\begin{enumerate}
    \item \textbf{Principled human behavior modeling and complementarity-alignment tradeoff characterization}: We introduce CGPR as a realistic, behaviorally-grounded model of human-AI interaction in AI-assisted decision making.
    Through this, we provide the first rigorous characterization of a fundamental complementarity-alignment tradeoff, proving why single-model approaches are inherently limited. Our models and analyses incorporate important behavioral factors, such as confidence and trust, into the core of the machine learning, rather than treating them as secondary considerations or empirical observations.

    \item \textbf{Practical adaptive AI ensemble framework}: We propose a human-centered adaptive AI ensemble to overcome tradeoffs like above. We offer a practical instantiation that trains specialist models for alignment and complementarity, and then efficiently toggles between them at test time through RRS, without requiring access to private or hard-to-measure human data.
    The benefits of our framework are validated through comprehensive experimentation.
    Our approach is flexible and extensible to a broad range of models and applications.

    \item \textbf{Extensive theoretical analyses and insights}: We provide extensive theoretical guarantees that bring a new level of rigor to human-centered machine learning. This includes precisely bounding the complementarity-alignment tradeoff (Theorem \ref{thm:local-tradeoff-formal}) as well as the performance gains of adaptive ensemble over optimal single model (Theorem \ref{thm:adaptive-gap-strong-convexity}), offering fresh  perspective on when and why adaptive AI is most useful in practice.
    These results establish a new class of human-centered theoretical tools.
\end{enumerate}

\paragraph{Other related work.}
Distinct from mainstream research on human-AI collaboration—which largely optimizes for complementarity or alignment, often overlooking nuanced behavioral dynamics—our work builds on and advances ``behavior-aware'' frameworks~\citep{bansal2021most,mahmood2024designing} by introducing more realistic cognitive modeling and rigorous theory. Parallel efforts on ``learning to defer"~\citep{madras2018predict,wilder2021learning,bondi2022role,dvijotham2023enhancing} investigate the \textit{division of labor} between humans and AI, typically assuming the AI can act as the final decision maker; in contrast, we focus on \textit{AI-assisted decision making} where the human always retains control and AI provides recommendations. Our approach uniquely formalizes the interplay between trust, confidence, and team performance, moving beyond simplified reliance models and establishing theoretically grounded methods for adaptive, human-centered teaming.

\section{The Complementarity-Alignment Tradeoff in Human-AI Collaboration}
\label{sec:tradeoff}

\subsection{Problem Setup}
In human-AI joint decision making, each instance of a decision making case is represented by features $\mathbf{x} \in \mathcal{X}$, and the objective is to predict an outcome $y \in \mathcal{Y}$. In our \textit{AI-assisted decision making} setup, an AI model $m$ provides a recommendation $y_m = m(\mathbf{x}; \theta_m)$ to a human decision maker (DM) $h$ who has an independent judgment $y_h = h(\mathbf{x}; \theta_h)$.
The final decision $d \in \mathcal{Y}$ is made by the human DM, based on their own initial judgment $y_h$ and the AI's recommendation $y_m$.

Typically, AI models are trained independently to minimize empirical risk over a dataset $\mathcal{D} = \{(\mathbf{x}_i, y_i)\}_{i=1}^N$, by optimizing
\(
    \theta_m = \argmin_{\theta_m'} \frac{1}{N} \sum_{i=1}^N \ell \left(m(\mathbf{x}_i; \theta_m'), y_i\right)
\),
where $\ell(\cdot)$ is an appropriate loss function. However, by optimizing for AI's independent accuracy instead of \textit{team performance}, this approach fails to account for how the human integrates the AI's recommendations into their final decision.
To model this human-AI interaction, we define a team decision making model $f(\cdot)$, which depends on the data instance, the AI recommendation, and the human DM's own judgment. The team loss optimization
is then formulated as:
\begin{equation*}
    \theta_m = \argmin_{\theta_m'} \frac{1}{N} \sum_{i=1}^N \ell \left(f(\mathbf{x}_i, m(\mathbf{x}_i; \theta_m'), h(\mathbf{x}_i; \theta_h)), y_i\right)
    \label{eq:team_loss}
\end{equation*}

\subsection{Modeling Human Behavior}
Optimizing human–AI team performance critically depends on accurately modeling how human DMs factor AI recommendations into their decisions. Prior work introduced a deterministic, confidence-thresholded reliance rule~\cite{mahmood2024designing}, where a human with confidence $\mathcal{C}^h(\mathbf{x})$ accepts the AI's prediction $y_m$ only when their own confidence falls below a threshold $\tau$.
Formally, this \emph{Confidence-Gated Reliance (CGR)} model takes the form:
\begin{equation}
\label{eq:CGR}
f_{\text{CGR}}(\mathbf{x}, m, h) =
\begin{cases}
y_h & \text{if } \mathcal{C}^h(\mathbf{x}) > \tau,\\
y_m & \text{otherwise}.
\end{cases}
\end{equation}
While analytically convenient, accumulating empirical evidence shows that human reliance on AI advice is often more nuanced and probabilistic, influenced significantly by human trust and observed AI behavior~\cite{chong2022human,schemmer2023appropriate,lu2021human,chiang2021you,narayanan2023does}.

To address this gap, we propose the \textit{Confidence-Gated Probabilistic Reliance (CGPR)} model. Let $\mathcal{D}_a = \{\mathbf{x} : \mathcal{C}^h(\mathbf{x}) > \tau\}$ and $\mathcal{D}_c = \{\mathbf{x} : \mathcal{C}^h(\mathbf{x}) \le \tau\}$
denote the ``alignment'' (high-confidence) and ``complementarity'' (low-confidence) regions, respectively. In $\mathcal{D}_a$, the human follows their own judgment, as in CGR. In $\mathcal{D}_c$, however, the DM relies on the AI prediction $y_m$ \textit{probabilistically}.
Recent evidence suggests that trust is disproportionately influenced by how often the AI contradicts the human when the human feels confident~\cite{lu2021human,grgic2022taking,wang2022will}.
We thus model the reliance probability $r$ to be governed by the perceived alignment between the human and the AI in the high-confidence region: $r = 1 - L_h(\mathcal{D}_a, m)$, with $L_h(\mathcal{D}_a, m) = \mathbb{E}_{\mathbf{x} \in \mathcal{D}_a}[\ell(m(\mathbf{x}), h(\mathbf{x}))]$ quantifying disagreement between the AI and the human.
We take $\ell(\cdot,\cdot)$ to be 0-1 loss here, so that $r \in [0,1]$ represents a probability.
Putting these components together, the CGPR team decision rule is:
\begin{equation}
\label{eq:team_model}
f(\mathbf{x}, m, h) =
\begin{cases}
y_h & \text{if } \mathcal{C}^h(\mathbf{x}) > \tau,\\[2pt]
y_m & \text{with prob. } r\;\text{if } \mathcal{C}^h(\mathbf{x}) \leq \tau,\\[2pt]
y_h & \text{otherwise.}
\end{cases}
\end{equation}

Here, we interpret trust as the human's willingness to follow the AI's recommendation, i.e., the reliance probability $r$ in CGPR, and thus use ``trust'' and ``reliance'' interchangeably. The CGPR model uniquely integrates human confidence, probabilistic reliance, and alignment-driven trust into a coherent and optimizable framework, significantly extending prior approaches and providing a principled foundation for (adaptive) human-centered AI.

\subsection{Optimizing Behavior-Aware AI}
Building on the behavior-aware AI framework,
we derive the optimal single-model AI model that directly optimizes team performance, accounting for the probabilistic and confidence-dependent nature of human reliance. Using Eq. \ref{eq:team_model}, the team loss can be decomposed as:
\begin{align}
\label{eq:team_loss}
L_{\text{team}}(\mathcal{D}, m,h)
=& \ L(\mathcal{D}_a, h)
   + L(\mathcal{D}_c, m) + \notag \\
&\;
   \big[ L(\mathcal{D}_c, h) - L(\mathcal{D}_c, m) \big] L_h(\mathcal{D}_a, m)
\end{align}

Since we cannot control a \textit{given} human DM $h$, we consider the human-only components $L_{a,h} \coloneqq L(\mathcal{D}_a, h)$ and $L_{c,h} \coloneqq L(\mathcal{D}_c, h)$ to be constant for optimization purpose. Therefore, the optimal, behavior-aware AI parameters are obtained by solving the following optimization problem:
\begin{equation}
    \label{eq:optimization_ba}
    \theta_{ba}^{*} = \argmin_{\theta_{ba}} L(\mathcal{D}_c, m) + \big[L_{c,h} - L(\mathcal{D}_c, m)\big] L_h(\mathcal{D}_a, m)
\end{equation}

\paragraph{Multiobjective optimization and the complementarity–alignment tradeoff.}
The behavior-aware optimization objective surfaces a tension between two desirable but conflicting goals: complementarity, achieved by minimizing $L(\mathcal{D}_c, m)$, and alignment, achieved by minimizing $L_h(\mathcal{D}_a, m)$.
Worse, their coupling in Eq. \ref{eq:optimization_ba} creates a compounding effect: as the model becomes better at complementing ($L(\mathcal{D}_c, m)$ decreases), the influence of alignment loss grows (via the term $[L_{c,h} - L(\mathcal{D}_c, m)]$), potentially increasing total team loss despite localized gains.
One could attempt to address this through standard multi-objective optimization, e.g., by minimizing a weighted combination $w \cdot L(\mathcal{D}_c, m) + (1-w) \cdot L_h(\mathcal{D}_a, m)$, thereby choosing a point on the Pareto frontier. However, this strategy merely shifts the compromise instead of resolving it.
The core issue is more fundamental: improving complementarity often entails diverging from human judgment patterns (in low-confidence regions), while improving alignment necessitates mimicking human behavior (in high-confidence regions).
These objectives could potentially be diametrically opposite.
This reveals a critical insight that optimizing a single model to balance both alignment and complementarity is not only inefficient but also often inherently suboptimal.

\subsection{Theoretical Analysis: One Model is Not Enough}
\label{sec:one_not_enough}
Next, we formalize the conflict between the two objectives---complementarity and alignment---as a direct tradeoff, demonstrating that under the single-model paradigm,  incremental improvements in alignment can lead to disproportionate degradation in complementarity, and vice versa.
For illustrative simplicity, our theoretical analysis is grounded in the geometry of the logistic loss with $\ell_2$-regularization, enabling fully data-dependent guarantees. Detailed discussion and proofs are deferred to Appendix.

\paragraph{Disentangling human and model factors.}
To effectively incorporate the contribution of human factors to the tradeoff, we first derive how model disagreement with human judgments in the alignment region is linked to the model's ground-truth prediction loss within the same region.

\begin{lemma}[Alignment Loss Sensitivity]
    \label{lem:alignment_sensitivity}
    The alignment loss $L_h(\mathcal{D}_a, m)$ can be decomposed as:
    \begin{equation*}
    \label{eq:alignment-loss-formula}
    L_h(\mathcal{D}_a,m)
    \;=\;
    \alpha \cdot L(\mathcal{D}_a,m)
    \;+\;
    (1-\alpha)\,\bigl[\,1 - L(\mathcal{D}_a,m)\bigr],
    \end{equation*}
    where $\alpha$ is human accuracy in the alignment region $\mathcal{D}_a$.
    Therefore, the sensitivity of alignment loss to changes in model's prediction loss can be quantified as:
    \begin{equation}
    \frac{\partial L_h(\mathcal{D}_a,m)}{\partial L(\mathcal{D}_a,m)}
    \;=\;
    2\alpha - 1.
    \end{equation}
\end{lemma}
\noindent This means attempting to improve model at mimicking human judgments effectively scales the required changes in the model's ability to predict ground truth by $(2\alpha-1)^{-1}$. For small $\alpha$, the alignment loss becomes largely insensitive to prediction loss, making it difficult to improve alignment by improving underlying ground-truth accuracy, and vice-versa.
Note that this decomposition is obtained under conditional independence assumption in binary classification setting, with a comprehensive analysis deferred to Appendix.

\paragraph{Quantifying the complementarity-alignment tradeoff.}
Let $L_a(\theta) \coloneqq L_h(\mathcal D_a,m(\cdot\,;\,\theta))$ and $L_c(\theta) \coloneqq L(\mathcal D_c,m(\cdot\,;\,\theta))$ be the alignment and complementarity objectives, each a logistic loss with $\ell_2$‐penalty.  Denote by $\theta^*_{m_a}$ and $\theta^*_{m_c}$ the corresponding minimizers, with $\theta$ denoting parameters for an arbitrary third model---including possibly $\theta_{ba}^*$ from Eq. \ref{eq:optimization_ba}---in the region of interest.
We consider a closed neighborhood $\mathcal R$ that contains all three points, and let $\lambda_{\max}^a$ and $\lambda_{\min}^c$ be the maximum and minimum eigenvalues of the Hessians of $L_a$ and $L_c$ over~$\mathcal R$.
Define the \textit{effective curvature ratio} $\lambda_r:=\lambda_{\min}^c/\lambda_{\max}^a$, which captures the ratio between how quickly alignment can improve and how slowly complementarity can degrade in the worst case. For notational convenience, we also define the human-accuracy factor as $\kappa \coloneqq 2\alpha - 1$, and quantify the distance from the specialist models to a given model as $d_c \coloneqq \|\theta - \theta^*_{m_c}\|$ and $d_a \coloneqq \|\theta - \theta^*_{m_a}\|$.

\begin{theorem}[Complementarity-Alignment Tradeoff]
\label{thm:local-tradeoff-formal}
For any $\theta\!\in\!\mathcal R$, define the local unit tradeoff
\begin{align}
  \mathcal{T}(\theta)
  &\;:=\;
  \lim_{\varepsilon\downarrow 0}
  \frac{L_c\!\bigl(\theta-\varepsilon\nabla L_a(\theta)\bigr)-L_c(\theta)}
       {-\,\bigl(L_a\!\bigl(\theta-\varepsilon\nabla L_a(\theta)\bigr)-L_a(\theta)\bigr)}
  \nonumber\\
  &\;=\;
  -\,\frac{\nabla L_c(\theta)^{\!\top}\nabla L_a(\theta)}{\|\nabla L_a(\theta)\|^2}, \nonumber
\end{align}
i.e.\ the instantaneous increase in complementarity loss per unit decrease in alignment loss when taking the steepest descent step for alignment.
For every $\theta\in\mathcal R$,
\begin{equation}
  \label{eq:tradeoff_bound_unit}
   \mathcal T(\theta)
   \;\;\ge\;\;
   \frac{\lambda_r}{\kappa}\;
   \frac{d_c}{d_a}
   \,\bigl(-\cos\!\phi(\theta)\bigr),
\end{equation}
where $\phi(\theta)$ is the angle between
$\nabla L_a(\theta)$ and $\nabla L_c(\theta)$.
Moreover, if the gradients remain sufficiently opposed in the descent direction in a neighborhood of $\theta^*_{m_a}$, i.e., $-\cos\phi(\theta) \ge c_0 > 0$ for some constant $c_0$, then $\lim_{\theta\to\theta^*_{m_a}} \mathcal T(\theta)=+\infty$ at rate $\Omega\left(\frac{1}{d_a}\right)$.
\end{theorem}

\noindent Theorem~\ref{thm:local-tradeoff-formal} starkly illustrates why a single model cannot satisfy both the objectives simultaneously.
In fact, whenever the specialist solutions are distinct and human accuracy is imperfect, it becomes mathematically impossible
to jointly achieve optimal alignment and complementarity. In the extreme case where the human's high confidence decisions are essentially guesses ($\alpha \rightarrow 0.5$), the tradeoff becomes unbounded.
Under realistic conditions, \emph{one model is not enough} for truly effective human-AI collaboration.

\section{Human-Centered Adaptive AI Ensembles}
The fundamental tradeoff established in last section suggests that we must move beyond the single model paradigm.
Fortunately, unlike classic tradeoffs that are uniform across the input space, adequate human behavior modeling guides us that the demands here vary \textit{systematically} across regions: trust is essential where humans are highly self-confident, while performance is essential where they are not.
Thus to effectively navigate the tension, we propose an adaptive ensemble of human-centered AI models.
The key idea is to leverage two specialist models, each dedicated to its own specific objective of alignment or complementarity, and to dynamically route decisions between them based on context,
thereby achieving the best of both worlds.

\subsection{Adaptive AI for Complementarity and Alignment}
Rather than forcing a single AI model onto the complementarity-alignment Pareto frontier, we can do better by training separate specialist AI models.
The two specialists in our case naturally stem from Eq. \ref{eq:optimization_ba} and the identified complementarity-alignment tradeoff. We thus train a \textit{complementary AI} $m_c$ parameterized by $\theta_{m_c}$, which minimizes prediction loss explicitly in the complementarity region ($\theta_{m_c}^{*} = \argmin L(\mathcal{D}_c, m_c)$), and an \textit{aligned AI} $m_a$ parameterized by $\theta_{m_a}$, which minimizes the disagreement with human judgments in the alignment region ($\theta_{m_a}^{*} = \argmin \,L_h(\mathcal{D}_a, m_a)$).

Having trained specialized AI models, our main challenge is to decide, for each instance, which model's advice to present to the human DM. An ideal solution, termed \textit{oracle routing}, with access to human confidence and threshold used within CGPR model (Eq. \ref{eq:team_model}) ``knows'' whether a test instance belongs to the alignment or complementarity region, and routes accordingly:
\begin{equation}
    \label{eq:ensemble-oracle}
    m_{\text{oracle}}(\mathbf{x}) =
    \begin{cases}
    m_a(\mathbf{x}) &\text{if} \; \mathbf{x} \in \mathcal{D}_a \;\; (\text{i.e.}, \mathcal{C}^h(\mathbf{x})>\tau), \\[3pt]
    m_c(\mathbf{x}) &\text{otherwise}.
    \end{cases}
\end{equation}

\subsubsection{Rational Routing Shortcut.}
The oracle routing requires directly observing the human DM's confidence and comparing it against their threshold, yet accurately estimating such human internal states is difficult, noisy, and context-dependent.
To bypass this complexity, we propose the \textit{Rational Routing Shortcut (RRS)} mechanism, which makes instance-wise routing decisions solely on the confidence estimates of the specialized models themselves.
Specifically, let
$\mathcal{C}^c(\mathbf{x})$
and $\mathcal{C}^a(\mathbf{x})$ be the respective confidence of the complementary and aligned specialists in their predictions. RRS routes
to the specialist with greater confidence:
\begin{equation}
    \label{eq:RRS_definition}
    m_\text{RRS}(\mathbf{x}) =
    \begin{cases}
        m_a(\mathbf{x}) & \text{if } \mathcal{C}^a(\mathbf{x}) \geq \mathcal{C}^c(\mathbf{x}), \\[4pt]
        m_c(\mathbf{x}) & \text{otherwise}.
    \end{cases}
\end{equation}

\paragraph{Why RRS intuitively works?}
RRS implicitly capitalizes on the intuition that higher model confidence is indicative of an instance falling within that model's designated region of expertise, using confidence as an \textit{implicit signal for region membership}.
Hence, RRS intelligently approximates oracle-like routing without ever directly observing human confidence.
Remarkably, this simple shortcut turns out to not only be intuitively appealing but also provably near-optimal \textit{for team decision making}:

\begin{theorem}[Near-Oracle Guarantee for RRS]
\label{thm:RRS}
Let $\mathcal D_a$ and $\mathcal D_c$ denote the alignment and complementarity regions, and let $m_a,m_c$ be the corresponding aligned and complementary specialists.
Suppose we are given real-valued confidence-estimation functions $\mathcal C^a,\,\mathcal C^c: \mathcal{X} \rightarrow [0,1]$.
Fix $\varepsilon\!\in\![0,1]$ and suppose that for every instance $\mathbf x$:

\begin{enumerate}[label=(\roman*),topsep=2pt,itemsep=1pt,leftmargin=13pt]
\item \textbf{(Calibration)}
$|\mathcal C^c(\mathbf x)-\mathbb{P}[m_c(\mathbf x)=y\mid\mathbf x]|\le\varepsilon$ and
$|\mathcal C^a(\mathbf x)-\mathbb{P}[m_a(\mathbf x)=h(\mathbf x)\mid\mathbf x]|\le\varepsilon$.

\item \textbf{(Estimator dominance in alignment region)}
If $\mathbf x\in\mathcal D_a$ then $\mathcal C^a(\mathbf x)\ge\mathcal C^c(\mathbf x)$.

\item \textbf{(Bounded sub-optimality outside alignment region)}
If $\mathbf x\in\mathcal D_c$ and $\mathcal C^a(\mathbf x)\ge\mathcal C^c(\mathbf x)$, then
$\mathbb{P}[m_a(\mathbf x)=y]\;\ge\;\mathbb{P}[m_c(\mathbf x)=y]-\varepsilon$.
\end{enumerate}

\noindent
Then the expected CGPR team accuracy achieved by the Rational Routing Shortcut satisfies
\[
\mathrm{Accuracy}_{\mathrm{RRS}}
\;\;\ge\;\;
\mathrm{Accuracy}_{\mathrm{Oracle}}
\;-\;\varepsilon.
\]
\end{theorem}

\subsection{Theoretical Analysis: The Benefit of Two Models}
To understand when and why our adaptive ensemble is provably advantageous, we analyze the performance gap between the (optimal) single-model solution and an oracle-style adaptive ensemble that uses the alignment specialist on $\mathcal{D}_a$ and the complementarity
specialist on $\mathcal{D}_c$. This yields a quantitative characterization of
when adaptive combination of two models is fundamentally better than one model.

Let $L_a(\theta)$ and $L_c(\theta)$ denote the losses restricted to the alignment and complementarity regions, respectively, with $\theta_{m_a}^*$ and $\theta_{m_c}^*$ being their corresponding minimizers. Define distance between specialist models as $D \coloneqq \|\theta_{m_a}^* - \theta_{m_c}^*\|$, and human accuracy factor $\kappa \coloneqq 2\alpha - 1$.

\begin{theorem}[Adaptive AI Performance Gain]
\label{thm:adaptive-gap-strong-convexity}
Assume the underlying prediction losses in the alignment and complementarity regions are $\mu$--strongly convex in $\theta$. Let $g_w(\theta) := w\,L_a(\theta) + (1-w)\,L_c(\theta)$ denote a weighted single-model surrogate objective. In particular, for $w = p := \mathbb{P}[\mathbf{x} \in \mathcal{D}_a]$, let $L_{\mathrm{single}}^* := \min_\theta g_p(\theta)$ be the loss of the best single model over the full population. The adaptive ensemble achieves
\(
L_{\mathrm{adapt}} \;=\; p\,L_a(\theta_{m_a}^*) + (1-p)\,L_c(\theta_{m_c}^*)
\). Then, the advantage of the adaptive ensemble can be lower bounded as:
\[
\Gamma_{\mathrm{team}} \;:=\; L_{\mathrm{single}}^* - L_{\mathrm{adapt}}
\;\ge\; \frac{\kappa \cdot \mu \cdot p(1-p) \cdot D^2}{2}.
\]
\end{theorem}

\noindent This shows how the benefit of using two models is not merely conceptual but can be precisely quantified in terms of human reliability, loss surface geometry, task mixture balance and specialist divergence. Specifically, the gap increases when:
(i) humans are more reliable in the aligned region ($\kappa$ large),
(ii) the loss landscape is ``well-conditioned'' or sharply curved ($\mu$ large),
(iii) the two regions occur with comparable frequency ($p(1-p)$ large), and
(iv) the alignment and complementarity optima are far apart ($D^2$ large).

\paragraph{Applicability to the behavior-aware team objective.}
Although the behavior-aware objective in Eq.~\ref{eq:team_loss} is non-convex due to the multiplicative term $(1 - L_h(\mathcal{D}_a,m))\,L(\mathcal{D}_c,m)$, the insight and bound of Theorem~\ref{thm:adaptive-gap-strong-convexity} remain indicative for the behavior-aware objective under mild regularity.
When human reliance $L_h(\mathcal{D}_a,m)$ varies smoothly with model confidence and increases with alignment performance, the objective locally behaves like a convex combination of two strongly convex region-specific objectives (up to bounded perturbations).
In this regime, the optimization landscape near the specialist solutions retains strong-convexity curvature, and the adaptive–single gap continues to scale quadratically with specialist separation.
This approximation is consistent with our empirical results (Fig.~\ref{fig:ca_gains_by_param}).

\subsection{Navigating Real-World Uncertainties}
In practice, it is difficult to tell whether a given instance belongs to the complementarity or alignment region.
We now (i) show how to train specialists under imperfect region knowledge, and (ii) quantify how this uncertainty propagates to test-time routing and resulting adaptive AI gains.

\subsubsection{Impact of uncertainty at training time.}
In real-world settings,
humans' self-confidence threshold $\tau$ required to evaluate whether an instance belongs to the complementarity or alignment region may not only be unknown to the AI model developer, but also vary across different DMs and across time.
To generalize to more realistic scenarios, we move beyond static human self-confidence threshold $\tau$ assumption, and instead model it as being drawn per-decision from a known (or estimated) distribution $f_T$.
With $\tau$ stochastic, we introduce the notion of \textit{probabilistic region membership} and expected region loss, which transforms our optimization problems into instance-weighted empirical risk minimization objectives.
In fact, we can neatly integrate this into our analysis by simply generalizing loss definitions. For example, the \textit{expected} complementarity loss becomes:
$L(\mathcal{D}_c,m) = \mathbb{E}_{\tau}\!\Bigl[
      \frac{1}{|\mathcal{D}|}
      \sum_{\mathbf{x}_i\in\mathcal{D}}
      \mathbb{I}\{\mathcal{C}_i^h\le\tau\}\,
      \ell\!\bigl(m(\mathbf{x}_i),y_i\bigr)
   \Bigr]$ $= \frac{1}{|\mathcal{D}|}
      \sum_{\mathbf{x}_i\in\mathcal{D}}
      \underbrace{\mathbb{E}[(\mathbf{x}_i, y_i) \in \mathcal{D}_c]}_{w_i^c} \,\,
      \ell\!\bigl(m(\mathbf{x}_i),y_i\bigr)$.
\begin{proposition}
\label{prop:comp-weight}
When human confidence thresholds are drawn from a known distribution with CDF $F_T(\cdot)$, the optimal instance weighting in the complementarity region is given by $w_i^c = 1 - F_T(\mathcal{C}_i^h)$.
\end{proposition}

\noindent A symmetric case can be made for expected alignment loss, with $w_i^a = F_T(\mathcal{C}_i^h)$. This weighting scheme naturally extends to both adaptive and behavior-aware AI training.

In line with insights from prior work,
this weighting is principled and flexible: in fact, assuming uniform $\tau$ yields $w^c_i = 1 - \mathcal{C}_i^h$, recovering previously effective heuristics as a special case.
Additionally, if
$F_T(\cdot)$ or individual confidence values are unknown, uniform distribution assumption and confidence-prediction models (learned from a small pilot) can provide robust proxies, enabling practical instance-weighted training.
Finally, while confidence calibration does influence the attainable gains, incorporating confidence into the AI optimization process leads to behavior-aware models that substantially outperform non-behavior-aware baselines~\cite{mahmood2024designing}.

\paragraph{Impact of uncertainty at test time.}
At deployment, uncertainty about region membership can lead to misrouting, i.e., instances assigned to the wrong specialist. Let $\bar \rho$ be the misrouting probability (expected fraction of misrouted instances), and $\mathcal{H}$ the average region entropy, measuring uncertainty. The expected adaptive team loss is
\begin{align}
L_{\text{adaptive}}
&= (1-\bar\rho)\big[L(\mathcal{D}_a, m_a) + L(\mathcal{D}_c, m_c)\big] \notag \\
&\quad +\ \bar\rho\big[L(\mathcal{D}_a, m_c) + L(\mathcal{D}_c, m_a)\big].
\end{align}
As we show in Appendix, the misrouting probability can actually be bounded by the entropy: $\bar\rho \leq \frac{\mathcal{H}}{2 \log 2}$. This, coupled with Theorem \ref{thm:adaptive-gap-strong-convexity}, helps us precisely quantify the impact of region uncertainty.

\begin{corollary}[Adaptive AI Performance Gain Under Uncertainty]
\label{cor:gain-uncertainty}
Let $\Gamma_{\text{team}}$ denote the performance gain of adaptive AI over the best single model, and define $p := \mathbb{P}[\mathbf{x} \in \mathcal{D}_a]$. Then, under uncertainty level $\mathcal{H}$,
\begin{equation}
\Gamma_{\text{team}} \geq \left(1 - \frac{\mathcal{H}}{2 \log 2}\right) \,\frac{\kappa \cdot \mu \cdot p(1-p) \cdot D^2}{2},
\end{equation}
\end{corollary}

This shows that the benefit of adaptive ensembles degrades gracefully as uncertainty increases, and even moderate certainty suffices for substantial gains. In summary, adaptive AI is robust to real-world uncertainty in both training and deployment, and principled weighting and entropy-based analyses enable reliable human-AI collaboration even when perfect information is unattainable.

\section{Evaluation on Simulated Data}
To rigorously evaluate our theoretical framework and substantiate the practical benefits of our proposed adaptive AI ensemble, we conducted controlled simulations on \textsf{College Admissions} data that we synthesize. This setting enables precise measurement of the complementarity-alignment tradeoff, which helps validate Theorem \ref{thm:local-tradeoff-formal}, and direct quantification of adaptive AI ensemble gains under varying conditions, which helps validate Theorem~\ref{thm:adaptive-gap-strong-convexity} and Corollary \ref{cor:gain-uncertainty}.

\subsection{Experimental Setup}
We mimic a committee making a binary admission decision ($\mathcal{Y}=\{+1, -1\}$) based on an applicant's Grade Point Average (GPA) and standardized test score. Consistent with some real settings, we assume applicants come from two subpopulations that differ in how well these features ($\mathbf{x} = \{x_\text{GPA}, x_\text{Score}\}$) indicate academic potential:
\begin{enumerate}
    \item \textit{Privileged Group:}
    We model the applicants from this group as residing within the feature space of alignment region ($\mathcal{D}_a$),
    where human evaluators are confident and typically accurate. Due to access to better schools and test prep, an applicant's test score is assumed to be the most predictive feature.
    \item \textit{Underprivileged Group:} We model these applicants as residing within the complementarity region ($\mathcal{D}_c$), where evaluators are less confident. Here, an applicant's GPA, reflecting long-term performance in their specific context, is assumed to be more predictive.
\end{enumerate}

\paragraph{Ground truth generation.}
For each instance, $x_\text{GPA}$ and $x_\text{Score}$ are sampled uniformly from $[0,1]$. The instance is assigned to the alignment region $\mathcal{D}_a$ with probability $p$ (default $p=0.5$). The ground truth label $y$ is then generated based on the region-specific predictive feature:
\begin{align*}
    \mathbf{x} \in \mathcal{D}_a: \; y &= \mathbb{I}[(0.5+\delta)x_\text{Score} + (0.5-\delta)x_\text{GPA} \ge 0.5] \\
    \mathbf{x} \in \mathcal{D}_c: \; y &= \mathbb{I}[(0.5+\delta)x_\text{GPA} + (0.5-\delta)x_\text{Score} \ge 0.5]
\end{align*}
Here, $\delta \in [0, 0.5]$ (default $\delta=0.25$) controls the feature importance and thus the divergence between the two regions' optimal decision boundaries, and consequently the specialist model divergence $D$.
This scenario provides an intuitive analog for the tension between alignment and complementarity, while offering precise control over key parameters.

\paragraph{Human DM behavior.}
We simulate human judgments $h(\mathbf{x})$ to be correct with probability $\alpha$ in $\mathcal{D}_a$ (default $\alpha=1$); human accuracy in $\mathcal{D}_c$ does not affect the CGPR dynamics and is set to a random baseline value when needed.
Human confidence $\mathcal{C}^h(\mathbf{x})$ is then drawn uniformly from $\pm 0.1$ around that region's accuracy. This couples confidence and accuracy in a noisy, realistic manner. Unless otherwise noted, group membership is treated as known for routing (i.e., oracle routing) to isolate the other theoretical factors.

\paragraph{AI training.}
We trained the following logistic regression-based AI models:
(i) \textit{Single AI}, trained to minimize the weighted objective $g_p(\theta)$ from Theorem \ref{thm:adaptive-gap-strong-convexity}, which collapses to standard AI trained for empirical risk minimization under our default settings;
(ii) \textit{Complementary AI}, trained on $\mathcal{D}_c$ using ground-truth labels $y$ and weighted loss $(1 - \mathcal{C}^h(\mathbf{x}))$;
(iii) \textit{Aligned AI}, trained on $\mathcal{D}_a$ using human predictions $h(\mathbf{x})$ as pseudo-labels and weighted loss $\mathcal{C}^h(\mathbf{x})$; and
(iv) \textit{Adaptive AI} ensemble that uses the \textit{pre-trained} specialists $m_a$ and $m_c$ and, dynamically selects the correct specialist based on the instance's known group membership ($G \in \{\mathcal{D}_a, \mathcal{D}_c\}$).

\begin{figure}[t]
    \centering
    \includegraphics[width=0.81\linewidth]{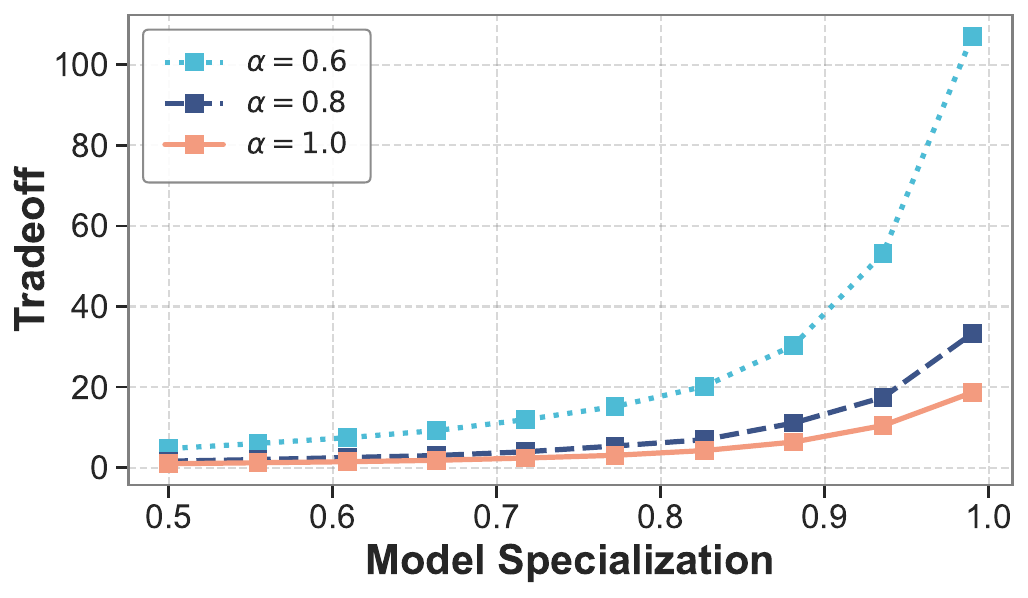}
    \caption{
    Complementarity-alignment tradeoff increases sharply with specialization, exacerbated by imperfect human DM in alignment region ($\alpha$),
    illustrating that a single model cannot simultaneously optimize for trust and performance.
    }
    \vspace{-10pt}
    \label{fig:ca_relative_tradeoff}
\end{figure}

\begin{figure*}
    \centering
    \includegraphics[width=\linewidth]{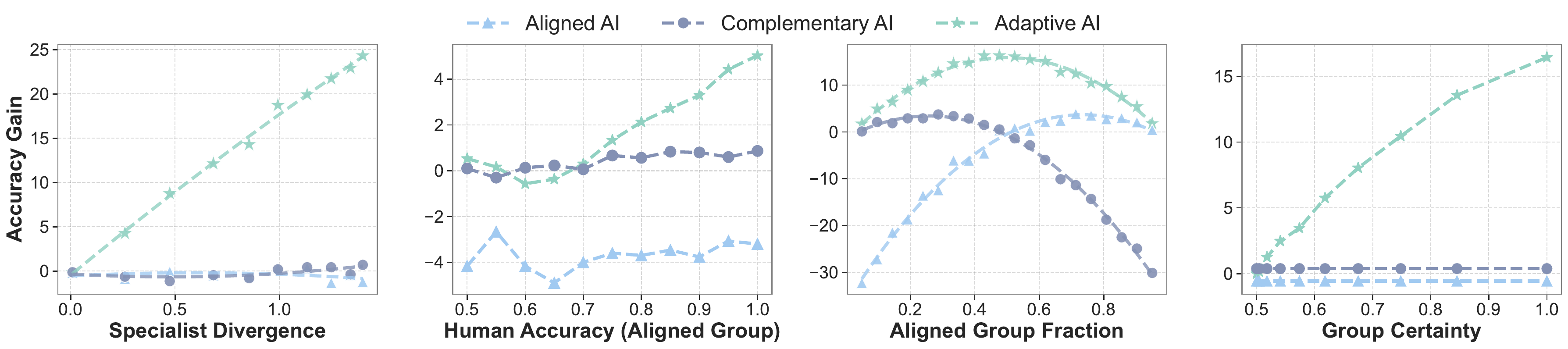}
    \caption{
    Accuracy gain ($\Gamma_\text{team}$) of adaptive AI ensemble over single AI on the \textsf{College Admissions} data.
    The plots (left to right) empirically validate that:
    (i) gain increases with specialist divergence ($\|\theta_{m_a}^* - \theta_{m_c}^* \|$);
    (ii) gain scales with human accuracy ($\alpha$) in alignment region (i.e., the aligned group's feature regime), reflecting the reliability factor $\kappa$;
    (iii) gain peaks when the task mixture is balanced (i.e., aligned group fraction $p \approx 0.5$), matching the concave $p(1-p)$ dependence; and
    (iv) gain increases linearly with group certainty ($1 - \mathcal{H} / 2\log 2$), demonstrating robustness to routing uncertainty.
    }
    \label{fig:ca_gains_by_param}
\end{figure*}

\subsection{Evaluating Complementarity-Alignment Tradeoff}

To probe the fundamental limitation of static models, we empirically measure the complementarity-alignment tradeoff defined in Theorem \ref{thm:local-tradeoff-formal}. Specifically, we assess how a small perturbation in model parameters---in the direction of the alignment specialist---impacts the model's loss in the complementarity region.
Formally, for a single-model optimum $\theta_{\mathrm{single}}^*$, a perturbation direction $v$ and small $\epsilon$, the tradeoff ratio is measured as:
\(
    \text{Tradeoff} = \frac{\Delta L(\mathcal{D}_c,\theta_{\mathrm{single}}^*+\epsilon v)}{-\Delta L(\mathcal{D}_a,\theta_{\mathrm{single}}^*+\epsilon v)},
\)
where $\Delta L(\mathcal{D}, \theta)$ denotes the change in loss on region $\mathcal{D}$ as the model parameters are nudged from $\theta_{\mathrm{single}}^*$.

To vary the degree of specialization, we increase the aligned group fraction $p$ from $0.5$ to $0.99$, thereby moving the optimal single model closer to the aligned specialist.
Figure~\ref{fig:ca_relative_tradeoff} shows that the empirical tradeoff grows sharply as $\theta_{\text{single}}^*$ approaches the aligned specialist, and
that the blow-up is faster when humans are less accurate in the aligned region (lower $\alpha$), confirming the $(1/\kappa)$ scaling predicted by theory.
These findings complement Theorem \ref{thm:local-tradeoff-formal}, and empirically establish that no single static model can simultaneously optimize alignment and complementarity when the regions are structurally distinct and human accuracy is imperfect.

\subsection{Evaluating Adaptive AI Gains}
We also empirically validate the adaptive AI performance gain predicted by our theoretical analysis (Theorem~\ref{thm:adaptive-gap-strong-convexity} and Corollary~\ref{cor:gain-uncertainty}). We systematically vary each of the four key factors in our bound---specialist divergence ($D$), human reliability ($\kappa$), task mixture balance ($p (1-p)$), and group uncertainty ($\mathcal{H}$)---and measure the resulting gain of the adaptive ensemble over the best single model. The results in Figure~\ref{fig:ca_gains_by_param} confirm all four theoretical dependencies.

\paragraph{Impact of specialist divergence.}
Here, we vary
$\delta \in [0,0.5]$, which directly controls the geometric divergence between the aligned and complementary
decision boundaries and hence the parameter distance $D$.
Figure~\ref{fig:ca_gains_by_param}(i) shows that even moderate divergence yields meaningful improvements to adaptive ensemble's accuracy gain, while large $\delta$ produces substantial adaptive advantages, confirming that \emph{specialist heterogeneity is a primary driver of adaptive benefit}.

\paragraph{Impact of human accuracy.}
Next, we vary the human accuracy in the alignment region, $\alpha \in [0.5,1]$, which modulates the reliability factor~$\kappa$ appearing in our theoretical bounds.
As shown in Figure~\ref{fig:ca_gains_by_param}(ii), the adaptive advantage increases monotonically with~$\alpha$. Higher human reliability strengthens the signal that the aligned specialist should follow, improving routing quality and magnifying the benefit of deferring to human expertise where appropriate. These results validate that \emph{gains scale with human reliability}.

\paragraph{Impact of task-mixture balance.}
Then, we vary the mixture proportion $p = \mathbb{P}[\mathbf{x}\in\mathcal{D}_a]$, which controls the relative size of the two subpopulations and determines the $p(1-p)$ scaling in our adaptive–single gap bound.
Figure~\ref{fig:ca_gains_by_param}(iii) confirms the predicted unimodal behavior: adaptive gains peak when the two regions are balanced ($p \approx 0.5$) and diminish when one region dominates the population.
When one task is rare, a single model suffers little penalty from specializing toward the larger region,
reducing the relative value of adaptivity.
This supports the theoretical insight that \emph{adaptive AI is most valuable when the environment contains multiple substantial and distinct regimes}.

\paragraph{Impact of region certainty.}
Finally, we vary the correctness of group labels, which determines how accurately the ensemble can assign each instance to the relevant specialist.
Figure~\ref{fig:ca_gains_by_param}(iv) illustrates that while perfect group information maximizes the adaptive AI advantage, substantial gains persist even when group assignments are noisy. This aligns with our theoretical observation that the benefit degrades smoothly with misrouting, and the \textit{ensemble remains robust as long as group identification is moderately reliable}.

\section{Evaluation on Real-World Data}
The controlled simulations above allowed us to isolate the structural factors predicted by our theory. To complement those findings, we examine whether the same qualitative patterns appear in a more complex perceptual domain where human competence, confidence, and uncertainty are empirically grounded. Instead of full-fledged real-world deployment, our goal is to assess whether the core mechanisms behind adaptive teaming continue to hold under realistic human variability. To this end, we evaluate on the \textsf{WoofNette} benchmark \cite{mahmood2024designing}, which was explicitly built to study human–AI collaboration under heterogeneous and instance-specific uncertainty.

\paragraph{Task and dataset.}
\textsf{WoofNette} is a 10-class classification benchmark built from ImageNet, containing five everyday objects (e.g., Gas Pump, Parachute) and five dog breeds (e.g., Australian Terrier, Dingo), with $9{,}446$ training and $4{,}054$ test images.
This structure naturally produces two regions: objects, where humans tend to be accurate and confident (approximate alignment region), and dog breeds, where human accuracy drops and confusion rises (approximate complementarity region). Importantly, instance-level variation within each class introduces additional uncertainty, making this a strong testbed for our adaptive framework.

\paragraph{Human DM behavior.}
Human predictions and self-reported confidence are estimated from a pilot study of $500$ images and $4{,}644$ annotations. We use inter-annotator agreement as a proxy for human confidence and train a lightweight ResNet-152 to generalize these confidence estimates to the full dataset, yielding $\widehat{\mathcal{C}}^h(\mathbf{x})$. Human predictions are simulated using empirical confusion matrices. Because region boundaries are not directly observable in natural perceptual tasks, each human's decision threshold $\tau$ is drawn from $\mathcal{U}[0.6,0.8]$, modeling individual variability and the bimodal confidence distribution observed in pilot data.

\paragraph{AI training.}
All AI models share a ResNet-50 backbone, initialized from ImageNet. To preserve room for human contribution, we cap training early when AI reaches human-level accuracy ($\approx 65\%$). We train the following AI paradigms:
(i) \textit{Standard AI}, trained on all data, without regard for human strengths or weaknesses; (ii) \textit{Aligned / Complementary AI}, trained with instance-weighted losses using $\widehat{\mathcal{C}}^h(\mathbf{x})$ to reflect the specialist regions;
(iii) \textit{Behavior-Aware AI}, trained to optimize team performance under the CGPR model (Eq. \ref{eq:optimization_ba}); and
(iv) \textit{Adaptive AI (Oracle, RRS)}, which combines previously trained specialists using either the oracle confidence labels (Adaptive (Oracle); Eq. \ref{eq:ensemble-oracle}) or the Rational Routing Shortcut (Adaptive (RRS); Eq. \ref{eq:RRS_definition}).

\begin{table}[t]
\centering
\begin{tabular}{lcc}
\toprule
\textbf{Paradigm} & \textbf{AI Accuracy} & \textbf{Team Accuracy} \\
\midrule
Standard AI         & $69.87_{\pm0.44}$ & $69.13_{\pm0.28}$ \\
Aligned AI          & $61.71_{\pm0.56}$ & $60.73_{\pm0.24}$ \\
Complementary AI    & $61.01_{\pm0.77}$ & $69.96_{\pm0.50}$ \\
Behavior-Aware AI   & $64.99_{\pm0.97}$ & $70.90_{\pm0.36}$ \\
Adaptive AI (Oracle)& $80.37_{\pm0.31}$ & $74.75_{\pm0.34}$ \\
Adaptive AI (RRS)   & $82.64_{\pm0.35}$ & $75.13_{\pm0.32}$ \\
\bottomrule
\end{tabular}
\caption{
Accuracy (\%) of independent AI and Human-AI team across different AI paradigms on the \textsf{WoofNette} benchmark. Human accuracy is $65.10_{\pm 0.27}$.
}
\label{tab:woofnette-results}
\end{table}

\subsection{Evaluation and Results}
For each test image, we simulate a team decision by drawing a threshold $\tau$, using the CGPR model to determine if the human defers to the AI, and recording the team accuracy over the entire test set (see Table~\ref{tab:woofnette-results}).

\paragraph{Adaptive AI ensemble outperforms baselines.}
The adaptive AI ensemble consistently achieves the highest team accuracy, outperforming both standard and behavior-aware baselines by clear margins. Notably, both Adaptive (Oracle) and Adaptive (RRS) achieve substantial improvements, with RRS variant performing virtually as well as the Oracle, despite having no access to human confidence at test time.

\paragraph{Gains realized despite weaker base AI models.}
The specialists used in the adaptive ensemble have lower individual accuracy than the standard AI, yet the adaptive team significantly outperforms the standard model. This reinforces the central insight of our theory and synthetic experiments, and underscores the value of intelligent combination and context-aware delegation over raw classifier performance.

\paragraph{Robustness to real-world constraints.}
This behavior-grounded experiment highlights that the adaptive AI framework, including its practical RRS instantiation, is robust to the uncertainties of a real-world task. Teaming gains persist even when region boundaries are instance-specific, human confidence is noisy, and AI predictions are imperfect.

\section{Conclusion}
This work advances the science of human-AI collaboration by introducing a principled and practical framework for adaptive teaming---one that is deeply informed by human behavioral realities, rigorous theory, and robust empirical validation. We formalize the complementarity-alignment tradeoff and show that the tension between performance-boosting and trust-building is an inherent limitation of single-model approaches, and provide the first tight theoretical characterization of this tradeoff. Our adaptive ensemble paradigm, powered by Confidence-Gated Probabilistic Reliance (CGPR) team decision making model and Rational Routing Shortcut (RRS)-based specialist selection, consistently achieves highest team performance in both synthetic and real-world experiments. Crucially, our results demonstrate that trustworthy, high-performing human-AI teams are possible only when AI systems dynamically align with human strengths and complement them where it matters most. We believe these advances contribute to the design of reliable, human-centered AI, enabling future systems that not only perform, but truly partner with, their human users.

\section*{Acknowledgments}
We thank the National Science Foundation for support under grants IIS-2229876 and IIS-2340209 for this work. Any opinions, findings, conclusions, or recommendations expressed here are those of the authors alone.

\bibliography{aaai2026}

\appendix
\setcounter{section}{0}
\renewcommand{\thesection}{\Alph{section}}
\renewcommand{\thesubsection}{\thesection.\arabic{subsection}}
\renewcommand\thefigure{A\arabic{figure}}
\renewcommand\thetable{A\arabic{table}}
\setcounter{figure}{0}
\setcounter{table}{0}

\section{Extended Related Work}

\paragraph{Human-AI Collaboration and Trust.}
Successful human-AI collaboration extends beyond merely optimizing AI accuracy, critically relying on human cognitive and behavioral dynamics such as trust, calibrated reliance, and interpretability \citep{yin2019understanding, bansal2019beyond, kolomaznik2024role}. Empirical studies consistently demonstrate that enhancing AI accuracy alone does not guarantee improved team performance, as unpredictable or counterintuitive AI behavior often erodes human trust and results in suboptimal team decision-making \citep{mayer2025human, chiang2021you, lu2024does}. Moreover, discrepancies between users' mental models of AI capabilities and actual AI performance can lead to both over- and under-reliance, exacerbating performance issues, especially in complex decision-making contexts \citep{vafa2024large, steyvers2025large, goh2024large}.

\paragraph{Complementarity vs. Alignment in AI Systems.}
Recent approaches in human-AI collaboration predominantly pursue the notion of complementarity, aiming to enhance overall decision quality by designing AI systems that selectively intervene on tasks or instances where human judgment is weak \citep{madras2018predict, lai2022human, bansal2021most, steyvers2022bayesian}. Although intuitively appealing, this paradigm frequently overlooks human trust dynamics, as disagreements between AI recommendations and confident human judgments—regardless of AI correctness—can significantly degrade trust and diminish reliance \citep{lu2021human, chong2022human, holstein2021designing}. Conversely, alignment-focused strategies, which prioritize human-AI agreement to foster trust, risk reinforcing incorrect human intuitions, thereby failing to leverage AI's potential for correcting human errors \citep{grgic2022taking, askell2021general, ouyang2022training}. Our work demonstrates, both theoretically and empirically, that single-model systems inherently struggle to simultaneously achieve high trust (alignment) and performance (complementarity), precisely quantifying this fundamental tradeoff for the first time.

\paragraph{Behavior-Aware AI and Optimization for Team Performance.}
To address these challenges, recent studies advocate for behavior-aware AI frameworks, which explicitly incorporate human behavioral models, such as trust calibration and confidence-based delegation, into the AI training process \citep{bansal2021most, mahmood2024designing}. However, existing behavior-aware approaches typically adopt simplistic or static assumptions regarding human decision-making processes, such as rational or fixed-threshold reliance models, limiting their robustness and flexibility in realistic and diverse human contexts \citep{vaccaro2024combinations, bansal2021most}. Our work advances this literature by introducing the Confidence-Gated Probabilistic Reliance model that dynamically integrates nuanced human trust and confidence dynamics, enabling more accurate modeling and optimization of team outcomes.

\paragraph{Ensemble and Multi-Objective Learning.}
Our adaptive, human-centered ensemble framework connects to a broad tradition of adaptive expert and ensemble methods \citep{jacobs1991adaptive, jordan1991hierarchies, shazeer2017outrageously, fedus2022switch, zoph2022st}. Classic mixture-of-experts and multi-task learning architectures have demonstrated the power of learning to dynamically combine specialized predictors for diverse tasks \citep{pfeiffer2021adapterfusion, momma2022multi, yu2020gradient}, and recent works have generalized this to settings involving multiple objectives and stakeholder preferences \citep{chakraborty2024maxmin, yang2024rewards, wang2025map}. 
However, these methods still generally reduce to one single (ensemble) model aiming to operate close to the desired point on Pareto frontier. Additionally, and perhaps more importantly, these ensemble methods rarely integrate human trust, confidence, or behavioral feedback as first-class objectives, nor do they address the unique challenges of human-AI team decision-making. ``Learning to defer'' approaches \citep{madras2018predict, mozannar2020consistent} have pioneered selective delegation, but typically focus on binary defer-or-predict mechanisms without the rich behavioral nuance of human confidence and alignment either.
\section{Theory}
\label{app:theory}

\subsection{Hessian eigenvalue bounds for logistic loss}
\label{app:hessian-bounds}

Many of this paper's bounds compare solutions of regularized logistic objectives via second-order control (e.g., Taylor expansions and strong convexity--type arguments). Rather than introducing abstract smoothness/strong-convexity constants, we work with explicit Hessian eigenvalue bounds that are (i) data-dependent and (ii) directly interpretable in terms of feature norms, empirical covariance, and (optionally) $\ell_2$ regularization.

\paragraph{Setup.}
Consider binary classification with data $\{(\mathbf{x}_i, y_i)\}_{i=1}^n$, where $\mathbf{x}_i \in \mathbb{R}^d$ and $y_i \in \{-1,+1\}$.
Let $\ell(z) = \log(1+\exp(-z))$ denote the scalar logistic loss and define the (possibly $\ell_2$-regularized) empirical objective
\begin{equation}
\label{eq:app-logistic-obj}
L(\theta)
\;=\;
\frac{1}{n}\sum_{i=1}^n \ell\!\big(y_i\langle \theta,\mathbf{x}_i\rangle\big)
\;+\;
\frac{\lambda}{2}\|\theta\|^2,
\qquad \lambda \ge 0 .
\end{equation}

\begin{assumption}[Data and logit regularity]
\label{assump:logistic-hessian}
We assume the following on a region of interest:
\begin{enumerate}[label=\textbf{A\arabic*.}, leftmargin=2.4em]
    \item \textbf{(Bounded features)} $\|\mathbf{x}_i\| \le B$ for all $i$.
    \item \textbf{(Empirical covariance lower bound)} $\frac{1}{n}\sum_{i=1}^n \mathbf{x}_i\mathbf{x}_i^\top \succeq \gamma I$ for some $\gamma>0$.
    \item \textbf{(Bounded logits on the region)} For all $\theta$ in the parameter region considered,
    \[
        \big|\langle \theta,\mathbf{x}_i\rangle\big| \le M
        \quad \text{for all } i\in\{1,\dots,n\}.
    \]
\end{enumerate}
\end{assumption}
\noindent Assumption above is standard and naturally holds under properly normalized data, constrained parameter space, or when regularization is present \cite{quinzan2023fast}.
Moreover, note that the setup above is flexible and encompasses several commonly studied models in machine learning, including logistic regression, one-layer neural networks and two-layer neural networks with fixed hidden layers.

\begin{lemma}[Scalar logistic curvature]
\label{lem:scalar-curv}
Let $\sigma(z)=1/(1+e^{-z})$. Then
\[
\ell''(z)=\sigma(z)\bigl(1-\sigma(z)\bigr)=\sigma(z)\sigma(-z),
\]
so $0<\ell''(z)\le \tfrac{1}{4}$ for all $z$. Moreover, on $|z|\le M$,
\[
\ell''(z)\ge c(M)
\;:=\;
\sigma(M)\sigma(-M)
\;>\;0.
\]
\end{lemma}
\begin{proof}
A direct calculation gives $\ell''(z)=\sigma(z)(1-\sigma(z))=\sigma(z)\sigma(-z)$.
This function is maximized at $z=0$ with value $1/4$, is strictly positive for all $z$, and on the interval $[-M,M]$ it attains its minimum at the endpoints $\pm M$, yielding the stated $c(M)$.
\end{proof}

\begin{lemma}[Hessian eigenvalue bounds for logistic loss]
\label{lem:hessian-bounds-logistic}
Let $\nabla^2 L(\theta)$ denote the Hessian of the (regularized) logistic loss.
\begin{enumerate}[label=(\roman*)]
    \item \textbf{(Uniform upper bound).} Under assumption A1,
    \begin{equation}
    \label{eq:app-hess-upper}
        \nabla^2 L(\theta)
        \;\preceq\;
        \left(\frac{B^2}{4}+\lambda\right) I.
    \end{equation}
    \item \textbf{(Lower bound on the bounded-logit region).} Under assumptions A2 and A3,
    \begin{equation}
    \label{eq:app-hess-lower}
        \nabla^2 L(\theta)
        \;\succeq\;
        \left(c(M)\gamma+\lambda\right) I.
    \end{equation}
\end{enumerate}
\end{lemma}
\begin{proof}
Recall that
\[
\nabla^2 L(\theta)
=
\frac{1}{n}\sum_{i=1}^n
\ell''\!\bigl(y_i\langle \theta,\mathbf{x}_i\rangle\bigr)\,
\mathbf{x}_i\mathbf{x}_i^\top
\;+\;
\lambda I .
\]
\textit{(i) Upper bound.}
By Lemma~\ref{lem:scalar-curv}, $\ell''(\cdot)\le 1/4$, hence
\[
\frac{1}{n}\sum_{i=1}^n
\ell''(\cdot)\,\mathbf{x}_i\mathbf{x}_i^\top
\;\preceq\;
\frac{1}{4n}\sum_{i=1}^n \mathbf{x}_i\mathbf{x}_i^\top.
\]
Using $\mathbf{x}_i\mathbf{x}_i^\top \preceq \|\mathbf{x}_i\|^2 I \preceq B^2 I$ from assumption A1 yields
$\frac{1}{4n}\sum_{i=1}^n \mathbf{x}_i\mathbf{x}_i^\top \preceq \frac{B^2}{4}I$,
and adding $\lambda I$ gives Eq. \eqref{eq:app-hess-upper}.

\noindent \textit{(ii) Lower bound.}
On the region where $|\langle \theta,\mathbf{x}_i\rangle|\le M$ for all $i$, Lemma~\ref{lem:scalar-curv} gives
$\ell''(y_i\langle \theta,\mathbf{x}_i\rangle)\ge c(M)$.
Therefore,
\[
\frac{1}{n}\sum_{i=1}^n
\ell''(\cdot)\,\mathbf{x}_i\mathbf{x}_i^\top
\;\succeq\;
c(M)\cdot \frac{1}{n}\sum_{i=1}^n \mathbf{x}_i\mathbf{x}_i^\top
\;\succeq\;
c(M)\gamma I,
\]
where the last step uses assumption A2. Adding $\lambda I$ yields Eq. \eqref{eq:app-hess-lower}.
\end{proof}

\paragraph{Notation.}
For convenience, we will refer to the resulting spectral bounds as
\begin{equation}
\label{eq:app-lambda-minmax}
\lambda_{\max}
\;:=\;
\frac{B^2}{4}+\lambda,
\qquad
\lambda_{\min}
\;:=\;
c(M)\gamma+\lambda.
\end{equation}
In subsequent proofs, $\lambda_{\max}$ plays the role of a smoothness constant and $\lambda_{\min}$ provides a (restricted) strong convexity modulus on the bounded-logit region.

\subsection{Proof of Lemma \ref{lem:alignment_sensitivity} (Extended)}
\label{app:proof-lemma2.1}

We formalize how \textit{alignment loss} (disagreement with the human) relates to the model's \textit{prediction loss} on the alignment region $\mathcal{D}_a$. This relationship is used in the main analysis to translate changes in predictive performance on $\mathcal{D}_a$ into changes in the reliance probability $r = 1 - L_h(\mathcal{D}_a,m)$ (and hence downstream team performance). For simplicity, we first present the cleanest closed form under binary labels and a mild conditional-independence condition, and then briefly discuss what changes in multiclass or correlated-error settings.

\paragraph{Binary setup.}
Assume $\mathcal{Y}=\{0,1\}$. Let $X$ denote a random instance conditioned on $X\in\mathcal{D}_a$. Let $Y\in\mathcal{Y}$ be the true label, $Y_H\in\mathcal{Y}$ the human decision, and $M(X)\in\mathcal{Y}$ the model prediction. In the high-confidence region $\mathcal{D}_a$, define
\begin{align*}
    \alpha &:= \mathbb{P}[Y_H = Y] &\text{(human accuracy)},\\
    L(\mathcal{D}_a,m) &:= \mathbb{P}[M(X)\neq Y] &\text{(model prediction loss)},\\
    L_h(\mathcal{D}_a,m) &:= \mathbb{P}[M(X)\neq Y_H] &\text{(alignment loss)}.
\end{align*}
All probabilities in this subsection are implicitly conditioned on $X\in\mathcal{D}_a$.

\begin{assumption}[Conditional independence of errors given the true label]
\label{assump:cond-indep}
Once the true label $Y = y$ is known, whether the human is correct or incorrect conveys no additional information about whether the model will err:
\begin{align}
    &\mathbb{P}[M(X) \neq y \mid Y_H = y,\, Y = y] \nonumber \\
    &\hspace{2em}= \mathbb{P}[M(X) \neq y \mid Y = y],
    \quad \forall\, y \in \{0,1\}.
\end{align}
\end{assumption}
This assumption is particularly plausible in $\mathcal{D}_a$, where instances are easy and human mistakes are typically sparse and noise-like (rather than systematically shared with model errors). Similar independence assumptions are common in prior human--AI works (e.g., \cite{bansal2021most}).

We now prove the decomposition stated in Lemma~\ref{lem:alignment_sensitivity} in binary setting under conditional independence assumption.

\begin{proof}[Proof of Lemma~\ref{lem:alignment_sensitivity}]
We condition on whether the human is correct:
\begin{align}
L_h(\mathcal{D}_a,m)
&= \mathbb{P}[M(X)\neq Y_H] \nonumber \\
&= \mathbb{P}[M(X)\neq Y_H \mid Y_H=Y]\cdot \mathbb{P}[Y_H=Y] \nonumber \\
&\quad + \mathbb{P}[M(X)\neq Y_H \mid Y_H\neq Y]\cdot \mathbb{P}[Y_H\neq Y]. \label{eq:lh-split}
\end{align}
By definition, $\mathbb{P}[Y_H=Y]=\alpha$ and $\mathbb{P}[Y_H\neq Y]=1-\alpha$.

\noindent \textit{Term 1: conditioning on $Y_H=Y$.}
When $Y_H=Y$, the event $\{M(X)\neq Y_H\}$ is exactly $\{M(X)\neq Y\}$, hence
\begin{align}
\mathbb{P}[M(X)\neq Y_H \mid Y_H=Y]
&= \mathbb{P}[M(X)\neq Y \mid Y_H=Y] \nonumber \\
&= \mathbb{P}[M(X)\neq Y] \nonumber\\
&= L(\mathcal{D}_a,m), \label{eq:term1}
\end{align}
where the second equality follows from Assumption~\ref{assump:cond-indep} (applied and then marginalized over $Y$).

\noindent \textit{Term 2: conditioning on $Y_H\neq Y$.}
Under binary labels, $Y_H\neq Y$ implies $Y_H = 1-Y$, so the event $\{M(X)\neq Y_H\}$ is equivalent to $\{M(X)=Y\}$. Therefore,
\begin{align}
\mathbb{P}[M(X)\neq Y_H \mid Y_H\neq Y]
&= \mathbb{P}[M(X)=Y \mid Y_H\neq Y] \nonumber \\
&= 1 - \mathbb{P}[M(X)\neq Y \mid Y_H\neq Y] \nonumber \\
&= 1 - \mathbb{P}[M(X)\neq Y] \nonumber\\
&= 1 - L(\mathcal{D}_a,m), \label{eq:term2}
\end{align}
where the penultimate step uses Assumption~\ref{assump:cond-indep}.

\noindent Substituting Eqs. \eqref{eq:term1}--\eqref{eq:term2} into Eq. \eqref{eq:lh-split} yields the decomposition
\begin{equation}
\label{eq:alignment-loss-formula-app}
L_h(\mathcal{D}_a,m)
\;=\;
\alpha \, L(\mathcal{D}_a,m)
\;+\;
(1-\alpha)\,\bigl[1 - L(\mathcal{D}_a,m)\bigr].
\end{equation}
Taking the derivative with respect to $L(\mathcal{D}_a,m)$ gives
\[
\frac{\partial L_h(\mathcal{D}_a,m)}{\partial L(\mathcal{D}_a,m)}
= 2\alpha - 1.
\]
\end{proof}

\paragraph{Extension to multiclass setting.}
For $|\mathcal{Y}|>2$, the identity used above,
\[
Y_H\neq Y \;\Longrightarrow\; \{M(X)\neq Y_H\}\equiv \{M(X)=Y\},
\]
no longer holds. Conditioning on whether the human is correct still yields
\begin{align}
L_h(\mathcal{D}_a,m)
&= \alpha\cdot \mathbb{P}[M(X)\neq Y \mid Y_H=Y] \nonumber\\
&\quad + (1-\alpha)\cdot \mathbb{P}[M(X)\neq Y_H \mid Y_H\neq Y], \label{eq:multiclass-split}
\end{align}
but the second conditional probability depends on \textit{which} wrong label $Y_H$ takes and how often the model matches that wrong label. Thus, without introducing additional confusion-matrix parameters describing the joint distribution of $(M(X),Y_H)$ given $Y$, there is no binary-style closed form.

\paragraph{Extension to correlated errors setting.}
If the human and model share systematic failure modes (e.g., particular subpopulations or visual artifacts), then
$\mathbb{P}[M(X)\neq y \mid Y_H=y,\, Y=y]$ can differ substantially from
$\mathbb{P}[M(X)\neq y \mid Y=y]$.
In this case, the decomposition in Eq. \eqref{eq:alignment-loss-formula-app} generally does not hold.
One can still express $L_h(\mathcal{D}_a,m)$ by conditioning on $(Y,Y_H)$, but additional dependence parameters are required to relate it to $L(\mathcal{D}_a,m)$.

\subsection{Proof of Theorem \ref{thm:local-tradeoff-formal}}
\label{app:proof-thm2.2}

We make precise the local \textit{complementarity--alignment tradeoff}: when taking an infinitesimal steepest-descent step for the alignment objective, the complementarity objective can increase, and the \textit{unit tradeoff} is controlled by (i) the angle between the two gradients and (ii) the relative distances to the two optima, scaled by curvature (Hessian eigenvalue) bounds.

\paragraph{Setup.}
Recall that $L_a(\theta)$ and $L_c(\theta)$ denote the (regularized) empirical logistic losses restricted to $\mathcal{D}_a$ and $\mathcal{D}_c$, respectively. Let
\[
\theta_{m_a}^* := \arg\min_{\theta} L_a(\theta),
\qquad
\theta_{m_c}^* := \arg\min_{\theta} L_c(\theta),
\]
and define distances
\[
d_a := \|\theta-\theta_{m_a}^*\|,
\qquad
d_c := \|\theta-\theta_{m_c}^*\|.
\]
Let $\phi(\theta)$ be the angle between $\nabla L_a(\theta)$ and $\nabla L_c(\theta)$, i.e.,
\[
\cos\phi(\theta)
:=
\frac{\nabla L_a(\theta)^\top \nabla L_c(\theta)}
{\|\nabla L_a(\theta)\|\,\|\nabla L_c(\theta)\|}.
\]
Let $\lambda_{\min}^a,\lambda_{\max}^a$ (resp.\ $\lambda_{\min}^c,\lambda_{\max}^c$) denote the Hessian eigenvalue lower/upper bounds for $L_a$ (resp.\ $L_c$) on the region of interest, as provided by Lemma~\ref{lem:hessian-bounds-logistic}.
Finally, when translating between alignment loss and prediction loss in $\mathcal{D}_a$, Lemma~\ref{lem:alignment_sensitivity} yields the scalar factor $\kappa=2\alpha-1$.

\paragraph{Step 1: Infinitesimal alignment-descent step.}
Fix $\theta$ and consider the steepest descent direction for $L_a$:
\[
\mathbf{d} := -\varepsilon \nabla L_a(\theta),
\qquad \varepsilon>0.
\]
Define $\Delta L_a(\varepsilon):=L_a(\theta+\mathbf{d})-L_a(\theta)$ and
$\Delta L_c(\varepsilon):=L_c(\theta+\mathbf{d})-L_c(\theta)$.

\paragraph{Step 2: Lower bound $\Delta L_c(\varepsilon)$.}
By the second-order Taylor lower bound using $\nabla^2 L_c(\cdot)\succeq \lambda_{\min}^c I$ (Lemma~\ref{lem:hessian-bounds-logistic}),
\begin{align}
L_c(\theta+\mathbf{d})
&\ge L_c(\theta) + \nabla L_c(\theta)^\top \mathbf{d}
    + \frac{\lambda_{\min}^c}{2}\|\mathbf{d}\|^2 \nonumber \\
\Rightarrow
\Delta L_c(\varepsilon)
&\ge -\varepsilon\,\nabla L_c(\theta)^\top \nabla L_a(\theta)
    + \frac{\lambda_{\min}^c}{2}\varepsilon^2 \|\nabla L_a(\theta)\|^2.
\label{eq:thm22-numerator}
\end{align}
Writing the inner product via the angle $\phi(\theta)$ gives
\begin{align}
\Delta L_c(\varepsilon)
&\ge -\varepsilon\,\|\nabla L_c(\theta)\|\,\|\nabla L_a(\theta)\|\cos\phi(\theta) \nonumber\\
    &\qquad+ \frac{\lambda_{\min}^c}{2}\varepsilon^2 \|\nabla L_a(\theta)\|^2.
\label{eq:thm22-numerator-angle}
\end{align}

\paragraph{Step 3: Relate $\|\nabla L_c(\theta)\|$ to $d_c$.}
Since $L_c$ is $\lambda_{\min}^c$-strongly convex on the region,
\begin{align}
\big(\nabla L_c(\theta)-\nabla L_c(\theta_{m_c}^*)\big)^\top(\theta-\theta_{m_c}^*)
&\ge \lambda_{\min}^c \|\theta-\theta_{m_c}^*\|^2. \nonumber
\end{align}
Using $\nabla L_c(\theta_{m_c}^*)=\mathbf{0}$ and Cauchy--Schwarz yields
\begin{equation}
\label{eq:thm22-gradLc-lower}
\|\nabla L_c(\theta)\|
\;\ge\;
\lambda_{\min}^c \|\theta-\theta_{m_c}^*\|
\;=\;
\lambda_{\min}^c d_c.
\end{equation}

\paragraph{Step 4: Upper bound $-\Delta L_a(\varepsilon)$ and relate $\|\nabla L_a(\theta)\|$ to $d_a$.}
To bound the denominator, note that the Taylor \textit{upper} bound with
$\nabla^2 L_a(\cdot)\preceq \lambda_{\max}^a I$ gives
\begin{align}
L_a(\theta+\mathbf{d})
&\le L_a(\theta) + \nabla L_a(\theta)^\top \mathbf{d}
    + \frac{\lambda_{\max}^a}{2}\|\mathbf{d}\|^2 \nonumber \\
\Rightarrow
-\Delta L_a(\varepsilon)
&\ge \varepsilon\|\nabla L_a(\theta)\|^2
   - \frac{\lambda_{\max}^a}{2}\varepsilon^2 \|\nabla L_a(\theta)\|^2.
\label{eq:thm22-denom}
\end{align}
For the infinitesimal ratio as $\varepsilon\downarrow 0$, the leading term is
$\varepsilon\|\nabla L_a(\theta)\|^2$.

\noindent We also need an upper bound on $\|\nabla L_a(\theta)\|$ in terms of $d_a$.
By $\lambda_{\max}^a$-smoothness (Lipschitz gradient),
\[
\|\nabla L_a(\theta)-\nabla L_a(\theta_{m_a}^*)\|
\le \lambda_{\max}^a \|\theta-\theta_{m_a}^*\|.
\]
Since $\nabla L_a(\theta_{m_a}^*)=\mathbf{0}$, we obtain
\begin{equation}
\label{eq:thm22-gradLa-upper}
\|\nabla L_a(\theta)\|
\;\le\;
\lambda_{\max}^a \|\theta-\theta_{m_a}^*\|
\;=\;
\lambda_{\max}^a d_a.
\end{equation}

\paragraph{Step 5: Take the ratio and send $\varepsilon\to 0^+$.}
Combining Eq. \eqref{eq:thm22-numerator-angle} and the leading-order denominator from Eq. \eqref{eq:thm22-denom}, we have
\begin{align}
\frac{\Delta L_c(\varepsilon)}{-\Delta L_a(\varepsilon)}
&\geq
\frac{
    \begin{aligned}
    &-\varepsilon\, \|\nabla L_c(\theta)\|\, \|\nabla L_a(\theta)\| \cos\phi \\
    &\quad + \dfrac{\lambda_{\min}^c}{2}\, \varepsilon^2 \|\nabla L_a(\theta)\|^2
    \end{aligned}
}{
    \varepsilon\, \|\nabla L_a(\theta)\|^2
} \nonumber \\
&=
- \frac{\|\nabla L_c(\theta)\|}{\|\nabla L_a(\theta)\|} \cos\phi
+ \frac{\lambda_{\min}^c}{2} \varepsilon
\label{eq:thm22-main-ratio}
\end{align}
Letting $\varepsilon\to 0^+$ yields the local unit tradeoff
\[
\mathcal{T}(\theta)
:=
\lim_{\varepsilon\to 0^+}\frac{\Delta L_c(\varepsilon)}{-\Delta L_a(\varepsilon)}
\;\ge\;
-\frac{\|\nabla L_c(\theta)\|}{\|\nabla L_a(\theta)\|}\cos\phi(\theta).
\]
Substituting the bounds Eq. \eqref{eq:thm22-gradLc-lower} and Eq. \eqref{eq:thm22-gradLa-upper} gives
\begin{equation}
\label{eq:thm22-final}
\mathcal{T}(\theta)
\;\ge\;
\frac{\lambda_{\min}^c}{\lambda_{\max}^a}\,
\frac{d_c}{d_a}\,
\bigl(-\cos\phi(\theta)\bigr).
\end{equation}

\paragraph{Optional $\kappa$ factor (alignment loss vs.\ prediction loss).}
The bound Eq. \eqref{eq:thm22-final} is stated in terms of the gradients of $L_a(\theta)$.
If $L_a$ is instantiated as the \textit{alignment loss} $L_h(\mathcal{D}_a,m)$ rather than the model prediction loss $L(\mathcal{D}_a,m)$, Lemma~\ref{lem:alignment_sensitivity} provides the scalar conversion
$\frac{\partial L_h(\mathcal{D}_a,m)}{\partial L(\mathcal{D}_a,m)}=\kappa$.
Consequently, the tradeoff per unit \textit{decrease in alignment loss} introduces a factor $1/\kappa$.

\paragraph{Divergence as $d_a\to 0$.}
If $\theta\to\theta_{m_a}^*$, then $d_a\to 0$. Unless the two optima coincide (so $d_c\to 0$ as well), $d_c$ remains bounded away from $0$ in a neighborhood of $\theta_{m_a}^*$. If additionally the gradients remain sufficiently opposed near $\theta_{m_a}^*$, i.e.,
\[
-\cos\phi(\theta)\;\ge\; c_0>0
\quad \text{for all $\theta$ in a neighborhood of $\theta_{m_a}^*$},
\]
then Eq. \eqref{eq:thm22-final} implies
\[
\mathcal{T}(\theta)
\;\ge\;
\frac{\lambda_{\min}^c}{\lambda_{\max}^a}\,\frac{d_c}{d_a}\,c_0
\;=\;
\Omega\!\left(\frac{1}{d_a}\right),
\]
so $\mathcal{T}(\theta)\to +\infty$ as $d_a\to 0$.

\subsection{Proof of Theorem \ref{thm:RRS}}
\label{app:rrs-proof}

We prove that the Rational Routing Shortcut (RRS) achieves expected CGPR team accuracy within $\varepsilon$ of the oracle routing strategy under the conditions of Theorem~\ref{thm:RRS}. The key point is that CGPR behaves identically on the alignment region $\mathcal{D}_a$, so it suffices to compare the two strategies on $\mathcal{D}_c$.

\paragraph{Setup.}
For any instance $\mathbf{x}$, define the conditional correctness probabilities
\begin{align*}
    p_c(\mathbf{x}) &:= \mathbb{P}\!\left[m_c(\mathbf{x}) = y \mid \mathbf{x}\right],\\
    p_a(\mathbf{x}) &:= \mathbb{P}\!\left[m_a(\mathbf{x}) = y \mid \mathbf{x}\right],
\end{align*}
and the reported confidences
\begin{align*}
    \mathcal{C}^c(\mathbf{x}) &:= \text{confidence reported by } m_c \text{ on } \mathbf{x},\\
    \mathcal{C}^a(\mathbf{x}) &:= \text{confidence reported by } m_a \text{ on } \mathbf{x}.
\end{align*}
Let $m_{\mathrm{oracle}}(\mathbf{x})$ denote the oracle specialist (Eq.~\ref{eq:ensemble-oracle}), and $m_{\mathrm{RRS}}(\mathbf{x})$ denote the RRS-selected specialist (Eq.~\ref{eq:RRS_definition}).
Finally, for $j\in\{a,c\}$ let $r_j := 1 - L_h(\mathcal{D}_a,m_j)$ denote the CGPR reliance probability when model $m_j$ is selected in $\mathcal{D}_c$.

\paragraph{Step 1: Reduce to comparing performance on $\mathcal{D}_c$.}
Under CGPR, the expected team accuracy $\mathrm{Accuracy}_S$ for any routing strategy $S\in\{\mathrm{oracle},\mathrm{RRS}\}$ can be written as
\begin{align}
&
\underbrace{
\mathbb{E}_{\mathbf{x}\in\mathcal{D}_a}\!\left[\mathbf{1}\{h(\mathbf{x})=y\}\right]
}_{\text{(always human in alignment region)}}
\nonumber\\
&\quad +
\underbrace{
\mathbb{E}_{\mathbf{x}\in\mathcal{D}_c}\!\left[
r_S\cdot \mathbf{1}\{m_S(\mathbf{x})=y\}
+(1-r_S)\cdot \mathbf{1}\{h(\mathbf{x})=y\}
\right]
}_{\text{(CGPR mixture in complementarity region)}} .
\label{eq:rrs-acc-decomp}
\end{align}
The $\mathcal{D}_a$ term is identical for both strategies (CGPR always uses the human in $\mathcal{D}_a$), so it suffices to compare the $\mathcal{D}_c$ term.

\paragraph{Step 2: Pointwise comparison of oracle vs.\ RRS on $\mathcal{D}_c$.}
For $\mathbf{x}\in\mathcal{D}_c$, the oracle always selects $m_c$, hence
\[
p_{\mathrm{Oracle}}(\mathbf{x}) := p_c(\mathbf{x}).
\]
RRS selects $m_a$ when $\mathcal{C}^a(\mathbf{x})\ge \mathcal{C}^c(\mathbf{x})$ and otherwise selects $m_c$, so define
\[
p_{\mathrm{RRS}}(\mathbf{x})
:=
\begin{cases}
p_a(\mathbf{x}), & \text{if } \mathcal{C}^a(\mathbf{x}) \ge \mathcal{C}^c(\mathbf{x}),\\
p_c(\mathbf{x}), & \text{otherwise}.
\end{cases}
\]
If $\mathcal{C}^c(\mathbf{x}) > \mathcal{C}^a(\mathbf{x})$, then RRS selects $m_c$ and
$p_{\mathrm{RRS}}(\mathbf{x})=p_{\mathrm{Oracle}}(\mathbf{x})$.
If instead $\mathcal{C}^a(\mathbf{x}) \ge \mathcal{C}^c(\mathbf{x})$, then by the bounded sub-optimality condition in Theorem~\ref{thm:RRS}(iii),
\[
p_a(\mathbf{x}) \ge p_c(\mathbf{x}) - \varepsilon,
\]
so $p_{\mathrm{RRS}}(\mathbf{x}) \ge p_{\mathrm{Oracle}}(\mathbf{x})-\varepsilon$.
Combining both cases, for all $\mathbf{x}\in\mathcal{D}_c$,
\begin{equation}
\label{eq:per-instance-gap}
p_{\mathrm{RRS}}(\mathbf{x})
\;\ge\;
p_{\mathrm{Oracle}}(\mathbf{x}) - \varepsilon.
\end{equation}
\noindent Note: the calibration condition in Theorem~\ref{thm:RRS}(i) is not needed for this inequality, and it serves to interpret $\mathcal{C}^a,\mathcal{C}^c$ as approximate correctness estimates.

\paragraph{Step 3: Reliance factors coincide ($r_{\mathrm{RRS}}=r_{\mathrm{Oracle}}$).}
By Theorem~\ref{thm:RRS}(ii), if $\mathbf{x}\in\mathcal{D}_a$ then $\mathcal{C}^a(\mathbf{x})\ge \mathcal{C}^c(\mathbf{x})$, so RRS selects $m_a$ on $\mathcal{D}_a$, matching the oracle's choice on $\mathcal{D}_a$.
Therefore,
\[
L_h(\mathcal{D}_a,m_{\mathrm{RRS}})
=
L_h(\mathcal{D}_a,m_a)
=
L_h(\mathcal{D}_a,m_{\mathrm{oracle}}),
\]
and hence $r_{\mathrm{RRS}}=r_{\mathrm{Oracle}}$. Denote their common value by $r$.

\paragraph{Step 4: Bound the expected accuracy gap on $\mathcal{D}_c$.}
Taking conditional expectation of the $\mathcal{D}_c$ term in Eq. \eqref{eq:rrs-acc-decomp} yields
\begin{align}
&\mathrm{Accuracy}_{\mathrm{Oracle}} - \mathrm{Accuracy}_{\mathrm{RRS}} \nonumber\\
&=
\mathbb{E}_{\mathbf{x}\in\mathcal{D}_c}\!\left[
r\cdot\Big(p_{\mathrm{Oracle}}(\mathbf{x})-p_{\mathrm{RRS}}(\mathbf{x})\Big)
\right].
\label{eq:rrs-gap-exact}
\end{align}
By Eq. \eqref{eq:per-instance-gap},
$p_{\mathrm{Oracle}}(\mathbf{x})-p_{\mathrm{RRS}}(\mathbf{x})\le \varepsilon$ pointwise on $\mathcal{D}_c$, so
\[
0 \le r\cdot\Big(p_{\mathrm{Oracle}}(\mathbf{x})-p_{\mathrm{RRS}}(\mathbf{x})\Big)
\le r\,\varepsilon \le \varepsilon.
\]
Taking expectation in Eq. \eqref{eq:rrs-gap-exact} gives
\[
\mathrm{Accuracy}_{\mathrm{Oracle}} - \mathrm{Accuracy}_{\mathrm{RRS}}
\le \varepsilon.
\]
Rearranging yields
\[
\mathrm{Accuracy}_{\mathrm{RRS}}
\;\ge\;
\mathrm{Accuracy}_{\mathrm{Oracle}} - \varepsilon,
\]
which completes the proof.
\hfill$\qed$

\subsection{Proof of Theorem \ref{thm:adaptive-gap-strong-convexity}}
\label{app:proof-thm3.6}

We lower bound the advantage of the adaptive ensemble over the best single model by combining (i) strong convexity of the regional prediction losses and (ii) a weighted ``two-centers'' quadratic minimization. The only subtlety is that the \textit{alignment} objective in $\mathcal{D}_a$ is an agreement-with-human loss, which we incorporate via the $\kappa$ factor from Lemma~\ref{lem:alignment_sensitivity}.

\paragraph{Step 0: Reduce to prediction losses and incorporate $\kappa$.}
In the alignment region $\mathcal{D}_a$, the alignment objective is
$L_a(\theta)=L_h(\mathcal{D}_a,m(\cdot;\theta))$.
By Lemma~\ref{lem:alignment_sensitivity}, for binary labels and under its conditions,
\begin{equation}
\label{eq:adapt-kappa-affine}
L_a(\theta)
\;=\;
(1-\alpha) \;+\; \kappa\,\widetilde L_a(\theta),
\end{equation}
where $\widetilde L_a(\theta):=L(\mathcal{D}_a,m(\cdot;\theta))$ and $\kappa:=2\alpha-1\in(0,1]$.
Consider the full-population single-model surrogate objective
\[
g(\theta) := p\,L_a(\theta) + (1-p)\,L_c(\theta),
\qquad p:=\mathbb{P}[\mathbf{x}\in\mathcal{D}_a].
\]
Substituting Eq. \eqref{eq:adapt-kappa-affine} gives
\begin{equation}
\label{eq:adapt-g-tilde}
g(\theta)
=
p(1-\alpha)
+
\underbrace{\Big(p\kappa\,\widetilde L_a(\theta) + (1-p)\,L_c(\theta)\Big)}_{=:~\widetilde g(\theta)}.
\end{equation}
Since $p(1-\alpha)$ is constant in $\theta$, minimizing $g(\theta)$ is equivalent to minimizing $\widetilde g(\theta)$, and
\begin{equation}
\label{eq:adapt-gap-tilde}
\min_{\theta} g(\theta) - L_{\mathrm{adapt}}
\;=\;
\min_{\theta} \widetilde g(\theta) - \widetilde L_{\mathrm{adapt}},
\end{equation}
where
\(
\widetilde L_{\mathrm{adapt}}
:=
p\kappa\,\widetilde L_a(\theta_{m_a}^*)
+
(1-p)\,L_c(\theta_{m_c}^*).
\)
Moreover, Eq. \eqref{eq:adapt-kappa-affine} is an affine transform with positive slope $\kappa$, so $L_a(\theta)$ and $\widetilde L_a(\theta)$ share the same minimizer, resulting in $\theta_{m_a}^*=\arg\min_\theta \widetilde L_a(\theta)$.

\paragraph{Step 1: Strong-convexity lower bounds.}
By assumption, the underlying \textit{prediction} losses in the two regions are $\mu$-strongly convex in $\theta$. In particular,
for all $\theta$,
\begin{align}
\widetilde L_a(\theta)
&\ge
\widetilde L_a(\theta_{m_a}^*) + \frac{\mu}{2}\,\|\theta-\theta_{m_a}^*\|^2,
\label{eq:adapt-sc-a}\\
L_c(\theta)
&\ge
L_c(\theta_{m_c}^*) + \frac{\mu}{2}\,\|\theta-\theta_{m_c}^*\|^2.
\label{eq:adapt-sc-c}
\end{align}
Multiplying Eq. \eqref{eq:adapt-sc-a} by $p\kappa$ and Eq. \eqref{eq:adapt-sc-c} by $(1-p)$ and summing yields, for all $\theta$,
\begin{align}
\widetilde g(\theta)
&\ge
\widetilde L_{\mathrm{adapt}}
+
\frac{\mu}{2}\Big(
p\kappa\,\|\theta-\theta_{m_a}^*\|^2
+
(1-p)\,\|\theta-\theta_{m_c}^*\|^2
\Big).
\label{eq:adapt-g-lower}
\end{align}

\paragraph{Step 2: Minimize the weighted two-center quadratic.}
Let $a,b\in\mathbb{R}^d$ and $w_1,w_2>0$. A standard identity gives
\begin{equation}
\label{eq:adapt-quad-identity}
\min_{\theta}
\Big(
w_1\|\theta-a\|^2 + w_2\|\theta-b\|^2
\Big)
=
\frac{w_1w_2}{w_1+w_2}\,\|a-b\|^2,
\end{equation}
attained uniquely at $\widehat\theta=(w_1 a+w_2 b)/(w_1+w_2)$.

\noindent Applying Eq. \eqref{eq:adapt-quad-identity} with
$a=\theta_{m_a}^*$, $b=\theta_{m_c}^*$, $w_1=p\kappa$, and $w_2=1-p$,
and writing $D:=\|\theta_{m_a}^*-\theta_{m_c}^*\|$, we obtain
\begin{equation}
\label{eq:adapt-quad-min}
\min_{\theta}
\Big(
p\kappa\,\|\theta-\theta_{m_a}^*\|^2
+
(1-p)\,\|\theta-\theta_{m_c}^*\|^2
\Big)
=
\frac{p\kappa(1-p)}{(1-p)+p\kappa}\,D^2.
\end{equation}
Since $\kappa\in(0,1]$ and $p\in[0,1]$, we have $(1-p)+p\kappa \le 1$, and therefore
\begin{equation}
\label{eq:adapt-coeff-lower}
\frac{p\kappa(1-p)}{(1-p)+p\kappa}\,D^2
\;\ge\;
p\kappa(1-p)\,D^2.
\end{equation}
\noindent\textit{Remark (Tightness).}
The exact coefficient in Eq. \eqref{eq:adapt-quad-min} can be written as
$p\kappa(1-p)\cdot \frac{1}{(1-p)+p\kappa}$, and the multiplicative factor
$\frac{1}{(1-p)+p\kappa}\ge 1$ quantifies how much larger the tight minimum can be relative to the simpler lower bound in Eq. \eqref{eq:adapt-coeff-lower}. We use Eq. \eqref{eq:adapt-coeff-lower} for a clean closed form.

\paragraph{Step 3: Conclude the performance gain bound.}
Taking $\min_{\theta}$ on both sides of Eq. \eqref{eq:adapt-g-lower} and using Eqs. \eqref{eq:adapt-quad-min}--\eqref{eq:adapt-coeff-lower} yields
\begin{align}
\min_{\theta}\widetilde g(\theta)
&\ge
\widetilde L_{\mathrm{adapt}}
+
\frac{\mu}{2}\,p\kappa(1-p)\,D^2.
\label{eq:adapt-min-tildeg}
\end{align}
Finally, by Eq. \eqref{eq:adapt-gap-tilde},
\begin{align*}
\Gamma_{\mathrm{team}}
&:=
L_{\mathrm{single}}^* - L_{\mathrm{adapt}}
=
\min_{\theta} \widetilde g(\theta) - \widetilde L_{\mathrm{adapt}} \\
&\ge
\frac{\mu}{2}\,p\kappa(1-p)\,D^2
=
\frac{\kappa\mu}{2}\,p(1-p)\,D^2.
\end{align*}
\hfill$\qed$

\subsection{Proof of Proposition \ref{prop:comp-weight}}
\label{app:proof-comp-weight}

The instance weight $w_i^c$ is the \textit{expected indicator} that $\mathbf{x}_i$ is routed to the complementarity region $\mathcal{D}_c$ under a randomized human confidence threshold.
Recall that $\mathbf{x}_i\in\mathcal{D}_c$ if and only if $\mathcal{C}_i^h \le \tau$.
Assume $\tau$ is supported on $[0,1]$ with CDF $F_T$ and density $f_T$.
Then
\begin{align*}
w_i^c
&:= \mathbb{E}_\tau\!\left[\mathbf{1}\{\mathbf{x}_i\in\mathcal{D}_c\}\right]
= \mathbb{E}_\tau\!\left[\mathbf{1}\{\mathcal{C}_i^h \le \tau\}\right] \\
&= \int_{0}^{1} \mathbf{1}\{\mathcal{C}_i^h \le \tau\}\, f_T(\tau)\, d\tau
= \int_{\mathcal{C}_i^h}^{1} f_T(\tau)\, d\tau \\
&= 1 - F_T(\mathcal{C}_i^h),
\end{align*}
which proves the claim.
\hfill$\qed$

\subsection{Proof of Corollary \ref{cor:gain-uncertainty}}
\label{app:proof-cor-3.7}

We analyze how uncertainty in region membership at deployment degrades the performance gain of adaptive AI.
The argument proceeds in three steps:
(i) express the adaptive loss under misrouting,
(ii) bound the misrouting probability via entropy,
and (iii) derive a rigorous additive degradation bound.
We then show how the multiplicative form stated in the main text follows under an additional assumption.

\paragraph{Setup.}
Let
$p(\mathbf{x}) := \mathbb{P}[\mathbf{x}\in\mathcal{D}_a \mid \mathbf{x}] \in [0,1]$
denote the (possibly learned) posterior probability that instance $\mathbf{x}$ belongs to the alignment region, and
$\mathbb{P}[\mathbf{x}\in\mathcal{D}_c \mid \mathbf{x}] = 1 - p(\mathbf{x})$.
Define the average instance-level binary entropy (in nats)
\begin{equation}
\label{eq:uncertainty-entropy}
\mathcal{H}
:=
\mathbb{E}_{\mathbf{x}\sim\mathcal{D}}
\Big[
-\,p(\mathbf{x})\log p(\mathbf{x})
- \big(1-p(\mathbf{x})\big)\log\big(1-p(\mathbf{x})\big)
\Big].
\end{equation}
We consider hard routing by maximum posterior:
$\mathbf{x}$ is assigned to $\mathcal{D}_a$ if $p(\mathbf{x}) \ge \tfrac{1}{2}$ and to $\mathcal{D}_c$ otherwise.
The per-instance misrouting probability is therefore
\[
\rho(\mathbf{x}) := \min\{p(\mathbf{x}),\,1-p(\mathbf{x})\},
\qquad
\bar{\rho} := \mathbb{E}_{\mathbf{x}\sim\mathcal{D}}[\rho(\mathbf{x})].
\]

\paragraph{Adaptive team loss under uncertainty.}
Let $m_a$ and $m_c$ denote the specialists trained for $\mathcal{D}_a$ and $\mathcal{D}_c$, respectively.
Under misrouting rate $\bar{\rho}$, the expected adaptive loss can be written as
\begin{align}
L_{\mathrm{adapt}}(\bar{\rho})
&=
(1-\bar{\rho})\big[L(\mathcal{D}_a,m_a)+L(\mathcal{D}_c,m_c)\big]
\nonumber\\
&\quad
+
\bar{\rho}\big[L(\mathcal{D}_a,m_c)+L(\mathcal{D}_c,m_a)\big].
\label{eq:uncertainty-Ladapt}
\end{align}
Define the ideal (oracle-routed) adaptive loss
\[
L_{\mathrm{adapt}}^{\mathrm{ideal}}
:=
L(\mathcal{D}_a,m_a)+L(\mathcal{D}_c,m_c),
\]
and the routing-induced excess loss
\[
\Delta
:=
\big[L(\mathcal{D}_a,m_c)+L(\mathcal{D}_c,m_a)\big]
-
\big[L(\mathcal{D}_a,m_a)+L(\mathcal{D}_c,m_c)\big].
\]
Then Eq. \eqref{eq:uncertainty-Ladapt} can be rewritten as
\begin{equation}
\label{eq:uncertainty-Ladapt-delta}
L_{\mathrm{adapt}}(\bar{\rho})
=
L_{\mathrm{adapt}}^{\mathrm{ideal}}
+
\bar{\rho}\,\Delta.
\end{equation}
Since each specialist is optimized for its own region,
$L(\mathcal{D}_a,m_a)\le L(\mathcal{D}_a,m_c)$ and
$L(\mathcal{D}_c,m_c)\le L(\mathcal{D}_c,m_a)$, implying $\Delta\ge 0$.

\paragraph{Bounding misrouting probability by entropy.}
We now relate $\bar{\rho}$ to $\mathcal{H}$.

\begin{lemma}[Misrouting is tightly bounded by entropy]
\label{lem:misrouting-le-entropy-opt}
For $p\in[0,1]$, let $\rho(p):=\min\{p,1-p\}$ and define the binary entropy
$H(p):=-p\log p-(1-p)\log(1-p)$ (nats), with $0\log 0:=0$.
Then, for all $p\in[0,1]$,
\[
\rho(p)\;\le\;\frac{H(p)}{2\log 2}.
\]
Consequently, with $\rho(\mathbf{x})=\rho(p(\mathbf{x}))$ and
$\mathcal{H}=\mathbb{E}[H(p(\mathbf{x}))]$,
\[
\bar{\rho} \;\le\; \frac{\mathcal{H}}{2\log 2}.
\]
Moreover, the constant $1/(2\log 2)$ is optimal, and equality holds at $p=\tfrac{1}{2}$.
\end{lemma}
\begin{proof}
By symmetry about $p=\tfrac{1}{2}$, it suffices to consider $p\in[0,\tfrac{1}{2}]$, where $\rho(p)=p$.
Define $\phi(p):=H(p)-2(\log 2)\,p$.
We have $\phi(0)=0$ and $H(\tfrac{1}{2})=\log 2$, hence $\phi(\tfrac{1}{2})=0$.
A direct calculation yields
\[
\phi''(p)=H''(p)=-\frac{1}{p(1-p)}<0 \quad \text{for } p\in(0,1),
\]
so $\phi$ is strictly concave. Since $\phi$ is concave and equals $0$ at the endpoints of $[0,\tfrac{1}{2}]$,
it satisfies $\phi(p)\ge 0$ on this interval. Therefore $H(p)\ge 2(\log 2)p$ for $p\in[0,\tfrac{1}{2}]$, i.e.,
$p\le H(p)/(2\log 2)$. By symmetry (replace $p$ with $1-p$), the same holds on $[\tfrac{1}{2},1]$.

For optimality, suppose $\rho(p)\le c\,H(p)$ holds for all $p$. Evaluating at $p=\tfrac{1}{2}$ gives
$\tfrac{1}{2}\le c\,\log 2$, so $c\ge 1/(2\log 2)$. Equality holds at $p=\tfrac{1}{2}$.
\end{proof}

\paragraph{Additive degradation bound.}
Let $L_{\mathrm{single}}^*:=\min_{\theta} g_p(\theta)$ denote the best single-model loss, and define
\[
\Gamma_{\mathrm{team}}(\mathcal{H})
:=
L_{\mathrm{single}}^* - L_{\mathrm{adapt}}(\bar{\rho}),
\]
with $\Gamma_{\mathrm{team}}(0):=
L_{\mathrm{single}}^* - L_{\mathrm{adapt}}^{\mathrm{ideal}}$.
From Eq. \eqref{eq:uncertainty-Ladapt-delta},
\begin{equation}
\label{eq:uncertainty-gamma-exact}
\Gamma_{\mathrm{team}}(\mathcal{H})
=
\Gamma_{\mathrm{team}}(0) - \bar{\rho}\,\Delta.
\end{equation}
To control $\Delta$, we use boundedness of the per-instance loss. For $0$--$1$ loss, $\ell\in[0,1]$ holds by definition.
For logistic loss, under bounded logits (e.g., bounded features and a bounded parameter region as assumed in the curvature subsection),
$\ell(z)=\log(1+e^{-z})$ is uniformly bounded, and one may rescale by a constant so that $\ell\in[0,1]$ without changing the optimizer.
Accordingly, assume $0\le \ell \le 1$, implying each region-average loss lies in $[0,1]$ and hence
\begin{equation}
\label{eq:uncertainty-delta-bound}
0 \le \Delta \le 1.
\end{equation}
Combining Eq. \eqref{eq:uncertainty-gamma-exact} and Eq. \eqref{eq:uncertainty-delta-bound} yields
\begin{equation}
\label{eq:uncertainty-additive}
\Gamma_{\mathrm{team}}(\mathcal{H})
\;\ge\;
\Gamma_{\mathrm{team}}(0) - \bar{\rho}.
\end{equation}
Finally, Lemma~\ref{lem:misrouting-le-entropy-opt} gives $\bar{\rho}\le \mathcal{H}/(2\log 2)$, and
Theorem~\ref{thm:adaptive-gap-strong-convexity} gives
$\Gamma_{\mathrm{team}}(0)\ge \frac{\kappa\mu}{2}p(1-p)D^2$.
Substituting into Eq. \eqref{eq:uncertainty-additive} yields the explicit additive lower bound:
\[
\Gamma_{\mathrm{team}}(\mathcal{H})
\;\ge\;
\frac{\kappa\mu}{2}p(1-p)D^2 \;-\; \frac{\mathcal{H}}{2\log 2}.
\]

\paragraph{Recovering the multiplicative form.}
The main text states a multiplicative degradation factor. This follows from Eq. \eqref{eq:uncertainty-gamma-exact} if the misrouting penalty is not too large relative to the ideal gain, namely if
\begin{equation}
\label{eq:uncertainty-rel-penalty}
\Delta \le \Gamma_{\mathrm{team}}(0).
\end{equation}
Under Eq. \eqref{eq:uncertainty-rel-penalty}, we have
\begin{align*}
\Gamma_{\mathrm{team}}(\mathcal{H})
&=
\Gamma_{\mathrm{team}}(0) - \bar{\rho}\,\Delta\\   
&\ge
\Gamma_{\mathrm{team}}(0) - \bar{\rho}\,\Gamma_{\mathrm{team}}(0)\\
&=
(1-\bar{\rho})\,\Gamma_{\mathrm{team}}(0).
\end{align*}
Combining with $\bar{\rho}\le \mathcal{H}/(2\log 2)$ and
$\Gamma_{\mathrm{team}}(0)\ge \frac{\kappa\mu}{2}p(1-p)D^2$ yields
\[
\Gamma_{\mathrm{team}}(\mathcal{H})
\;\ge\;
\left(1-\frac{\mathcal{H}}{2\log 2}\right)\frac{\kappa\mu}{2}\,p(1-p)\,D^2
\]
\hfill$\qed$

\section{Experiments}
All experiments are implemented in Python using TensorFlow and scikit-learn with default hyperparameters unless otherwise noted. Training and evaluation are conducted on systems equipped with NVIDIA A10 GPUs.

\begin{figure}
    \centering
    \includegraphics[width=\linewidth]{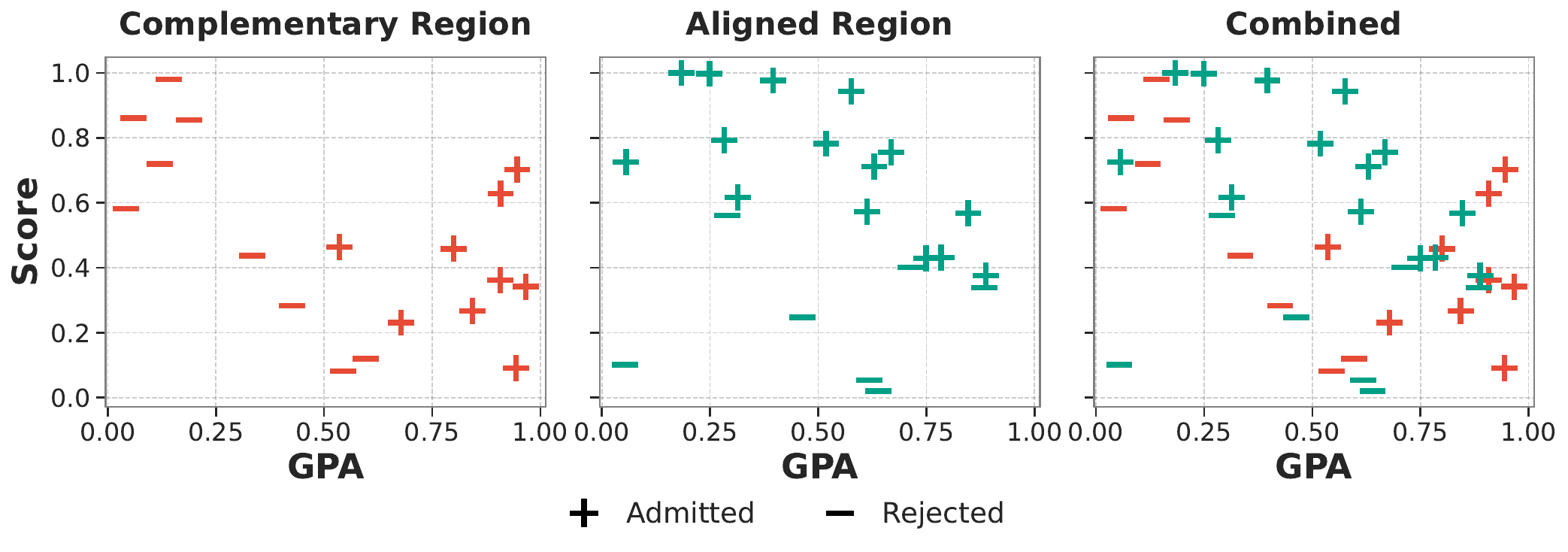}
    \caption{Samples from \textsf{College Admission} dataset with $p=0.5$ and $\delta = 0.25$.}
    \label{fig:ca_data}
\end{figure}

\begin{figure}
    \centering
    \includegraphics[width=\linewidth]{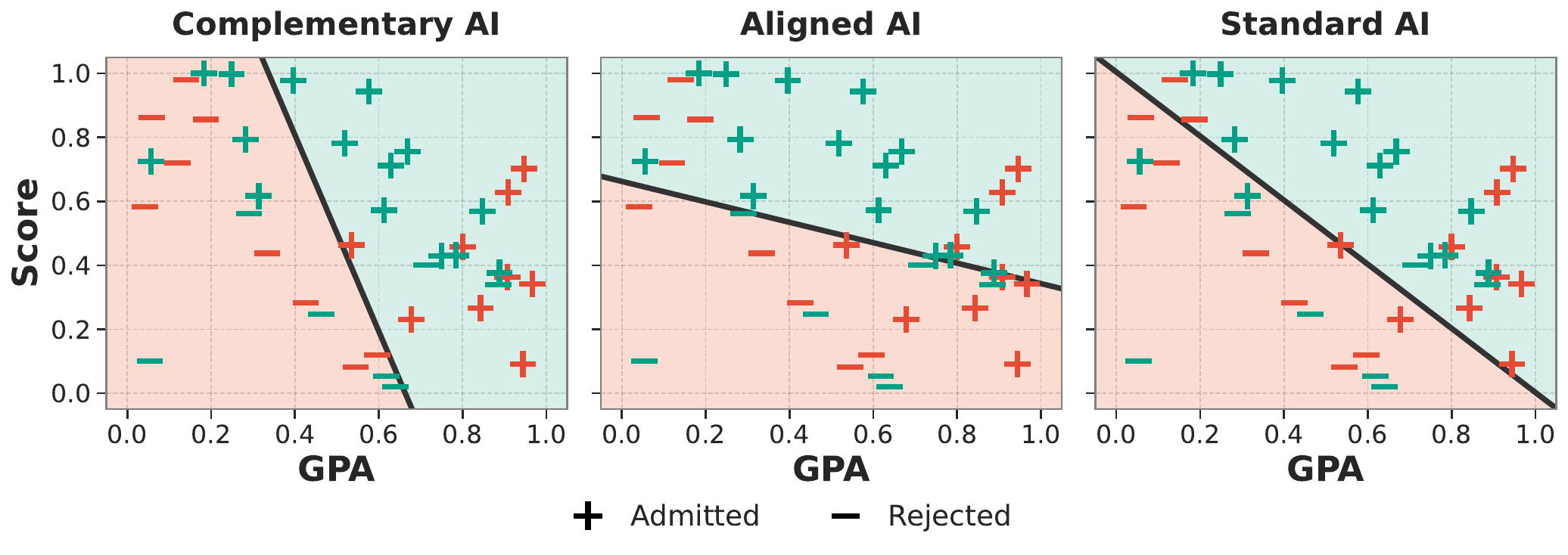}
    \caption{Visualization of decision boundary of standard and specialist logistic regression models on \textsf{College Admission} dataset with $p=0.5$ and $\delta = 0.25$.}
    \label{fig:ca_decision_boundary}
\end{figure}
\subsection{\textsf{College Admission}}
In this dataset, we mimic the college admission scenario, where decision makers need to determine whether to admit an applicant to college (i.e., $\mathcal{Y}=\{+1, -1\}$, $+1$ represents admitted while $-1$ represents rejected), given two features of the applicant---their Grade Point Average (i.e., ``\textsc{Gpa}'') and their standardized test scores (i.e., ``\textsc{Score}'').
Applicants belong to either a privileged (aligned) or underprivileged (complementary) group, reflecting regions where different signals are more informative and where human confidence may vary.
The dataset is constructed so that \textsc{Score} is more predictive of the admission outcome for privileged applicants, while \textsc{Gpa} is more predictive for underprivileged applicants. Privileged applicants, with access to better preparation and the ability to retake tests, are more likely to have a representative \textsc{Score}, making it a reliable basis for their admission decision. Underprivileged applicants, in contrast, better demonstrate their abilities through school-specific \textsc{Gpa}.

For each of the $n$ applicants, the values of $x_{\text{GPA}}$ and $x_{\text{Score}}$ are sampled uniformly from $[0,1]$. Each applicant is assigned to the privileged group with probability~$p$. The ground truth label is then generated based on group membership:
in the alignment region, $y = \mathbb{I}[(0.5+\delta)\,x_\text{Score} + (0.5-\delta)\,x_\text{GPA} \ge 0.5]$; in the complementarity region, $y = \mathbb{I}[(0.5+\delta)\,x_\text{GPA} + (0.5-\delta)\,x_\text{Score} \ge 0.5]$.
The parameter $\delta \in [0,0.5]$ controls the separation between the two regions' optimal decision boundaries, while $p$ governs how closely a single model approximates one of the two specialists. Our defaults are $p=0.5$ (balanced groups) and $\delta=0.25$ (moderate feature dominance per group). Figures~\ref{fig:ca_data} and~\ref{fig:ca_decision_boundary} illustrate this setting.

\paragraph{Generalization across model families.}
\begin{figure}
    \centering
    \includegraphics[width=0.85\linewidth]{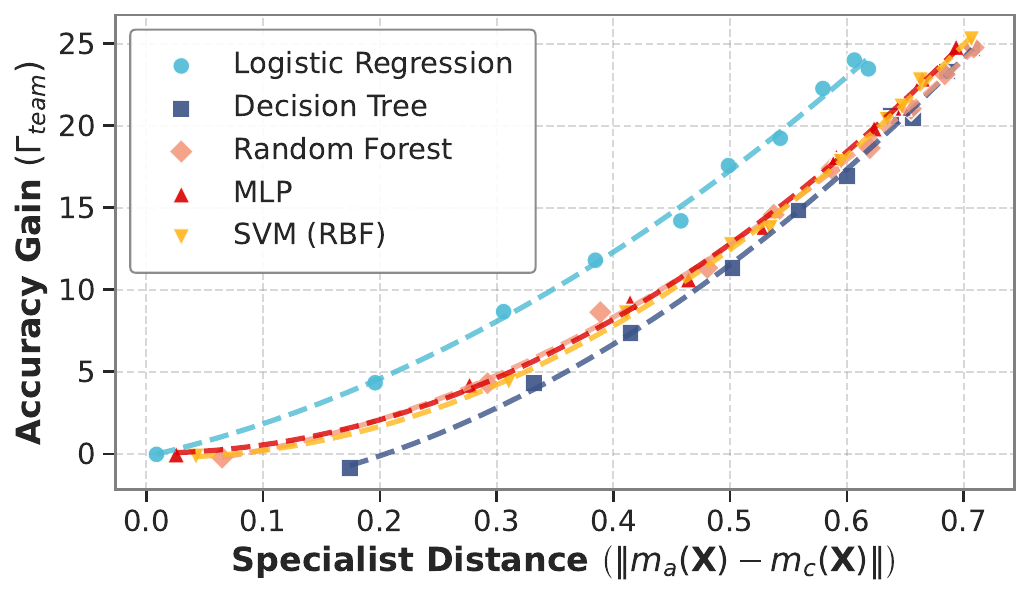}
    \caption{Adaptive AI gains extend beyond logistic regression and grow quadratically with specialist divergence across varied models.}
    \label{fig:relative_tradeoff_extended}
\end{figure}

We verify that the benefits of adaptive AI are not confined to the linear setting assumed in our theoretical analysis.
Using the same setting as Figure~\ref{fig:ca_gains_by_param}a in the main text, we evaluate the accuracy gap between adaptive AI and standard AI for four additional model families: Decision Tree, Random Forest, Multilayer Perceptron, and Support Vector Machine (Radial Basis Function kernel). We continue to vary $\delta$ to induce different levels of specialist divergence, but measure divergence via the change in model predictions, since the notion of model ``parameters'' is not straightforward for all architectures.
Figure~\ref{fig:relative_tradeoff_extended} shows that, across all five model families, accuracy gains grow quadratically with specialist divergence, consistent with the theoretical insights derived for the logistic regression setting.

\paragraph{Robustness to multi-objective model weight selection.}
Theorem~\ref{thm:adaptive-gap-strong-convexity} defines the optimal single-model baseline as
$L^*_{\text{single}}=\min_{\theta}\, g_p(\theta)$, i.e., the population-consistent
weighted surrogate with $w=p:=\mathbb{P}[x\in\mathcal{D}_a]$.
A natural concern is whether the adaptive AI advantage is an artifact of this particular weighting. To address this, we broaden the single-model comparison class by training a family of \emph{fixed-weight} models
$\theta_w^* \in \arg\min_{\theta}\, g_w(\theta)$ for
$w\in\{0,\, 0.25,\, 0.5,\, 0.75,\, 1\}$, where $w=0$ and $w=1$ correspond to the
complementary and aligned specialists, respectively, and $w=p$ recovers the
standard AI baseline.

Figure~\ref{fig:ca_weight_sweep} (top) shows accuracy gains over the standard model ($w{=}p$) across varying aligned group fractions~$p$. The largest single-model accuracy gains occur at the endpoints $w \in \{0, 1\}$, corresponding to pure specialists. This may appear counterintuitive: when $p = 0.75$, one might expect a model trained with $w = 0.75$ to outperform a specialist with $w = 1$. However, because the two groups' optimal boundaries are sufficiently distinct, soft interpolation via intermediate weights produces boundaries that are suboptimal for \emph{both} groups. Hard commitment to one group's boundary is more effective for that group, and adaptive AI exploits this by routing each instance to the appropriate specialist.

Figure~\ref{fig:ca_weight_sweep} (bottom) confirms that adaptive AI also achieves lower log-loss (i.e., the actual training objective) compared to every single model. Under log-loss, deviating from the population-consistent weight $w{=}p$ does not improve the single-model objective on the mixture distribution, reaffirming the baseline choice in Theorem~\ref{thm:adaptive-gap-strong-convexity}. Notably, the benefits of intermediate weights are restored under this metric: for example, at $p=0.25$, single AI with $w{=}0.25$ outperforms the complementary specialist ($w{=}0$). Taken together, these results demonstrate that the adaptive AI advantage persists regardless of the single-model weighting and cannot be recovered by tuning a fixed-weight surrogate.

\begin{figure}
    \centering
    \begin{subfigure}
        \centering
        \includegraphics[width=0.86\linewidth]{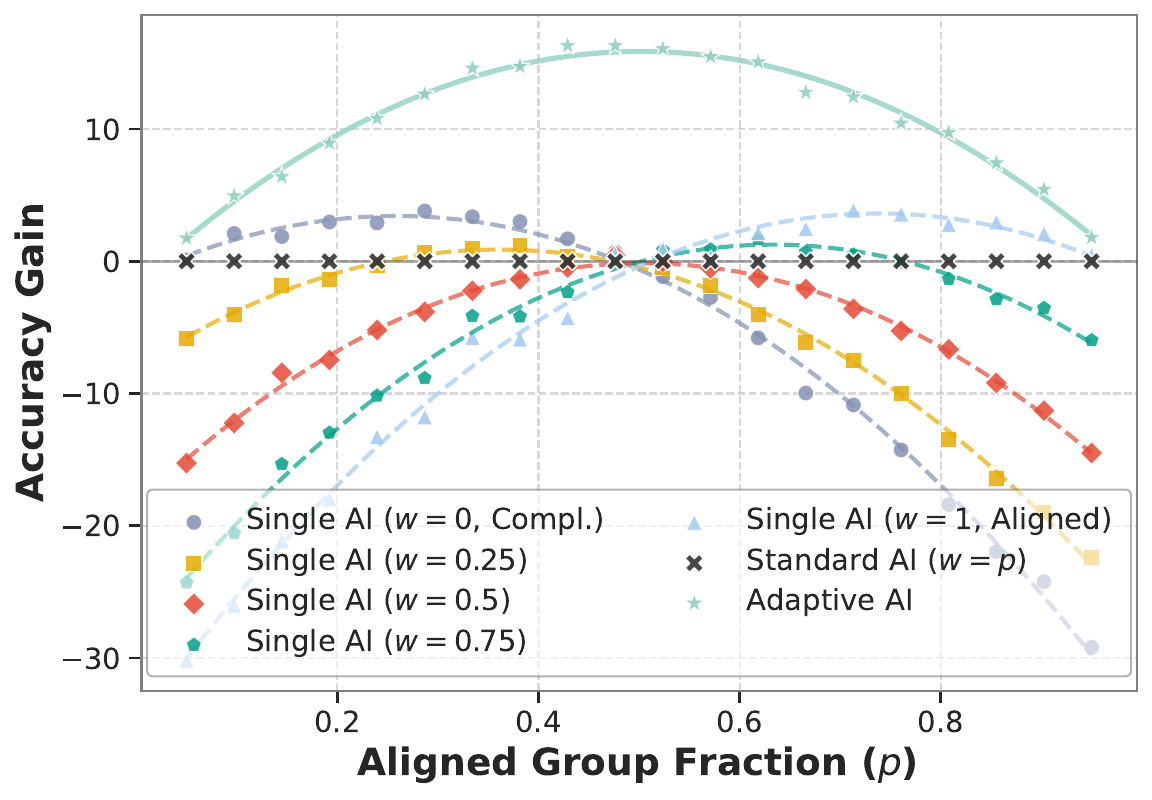}
        \label{fig:ca_weight_sweep_acc}
    \end{subfigure}
    \hfill
    \begin{subfigure}
        \centering
        \includegraphics[width=0.86\linewidth]{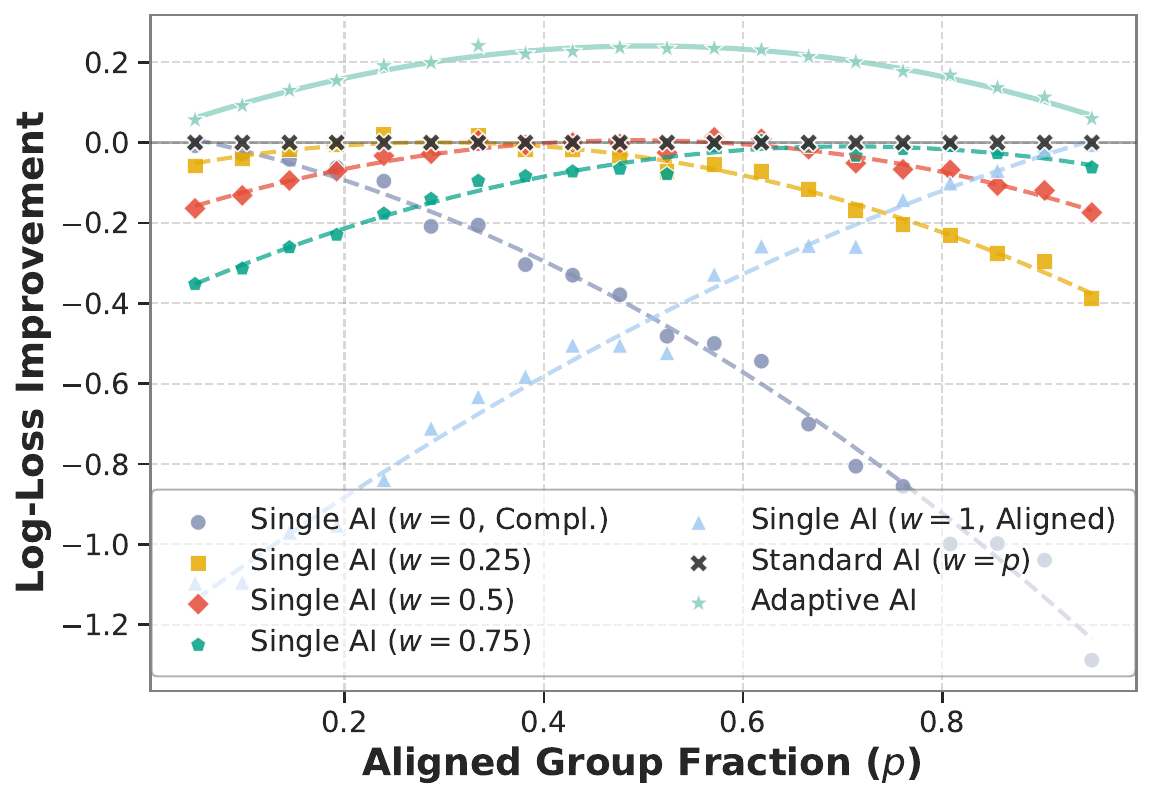}
        \label{fig:ca_weight_sweep_logloss}
    \end{subfigure}
    \caption{
    Adaptive AI gains persist across the full family of fixed-weight single-model surrogates.
    \textbf{Top:} accuracy gains over standard AI. Pure specialists ($w{\in}\{0,1\}$) can outperform intermediate weights, yet adaptive AI consistently dominates.
    \textbf{Bottom:} log-loss improvement (positive values indicate lower loss). Fixed-weight deviations do not improve upon the
    population-consistent objective, confirming $w=p$ as the appropriate single-model baseline; adaptive AI achieves further improvement by routing to specialists.}
    \label{fig:ca_weight_sweep}
\end{figure}

\subsection{\textsf{WoofNette}}

The \textsf{WoofNette} dataset \cite{mahmood2024designing} comprises images of five easily recognizable objects and five challenging dog breeds.
Sample training images are shown in Figure \ref{fig:woofnette_data}.

\begin{figure*}
    \includegraphics[width=0.81\linewidth]{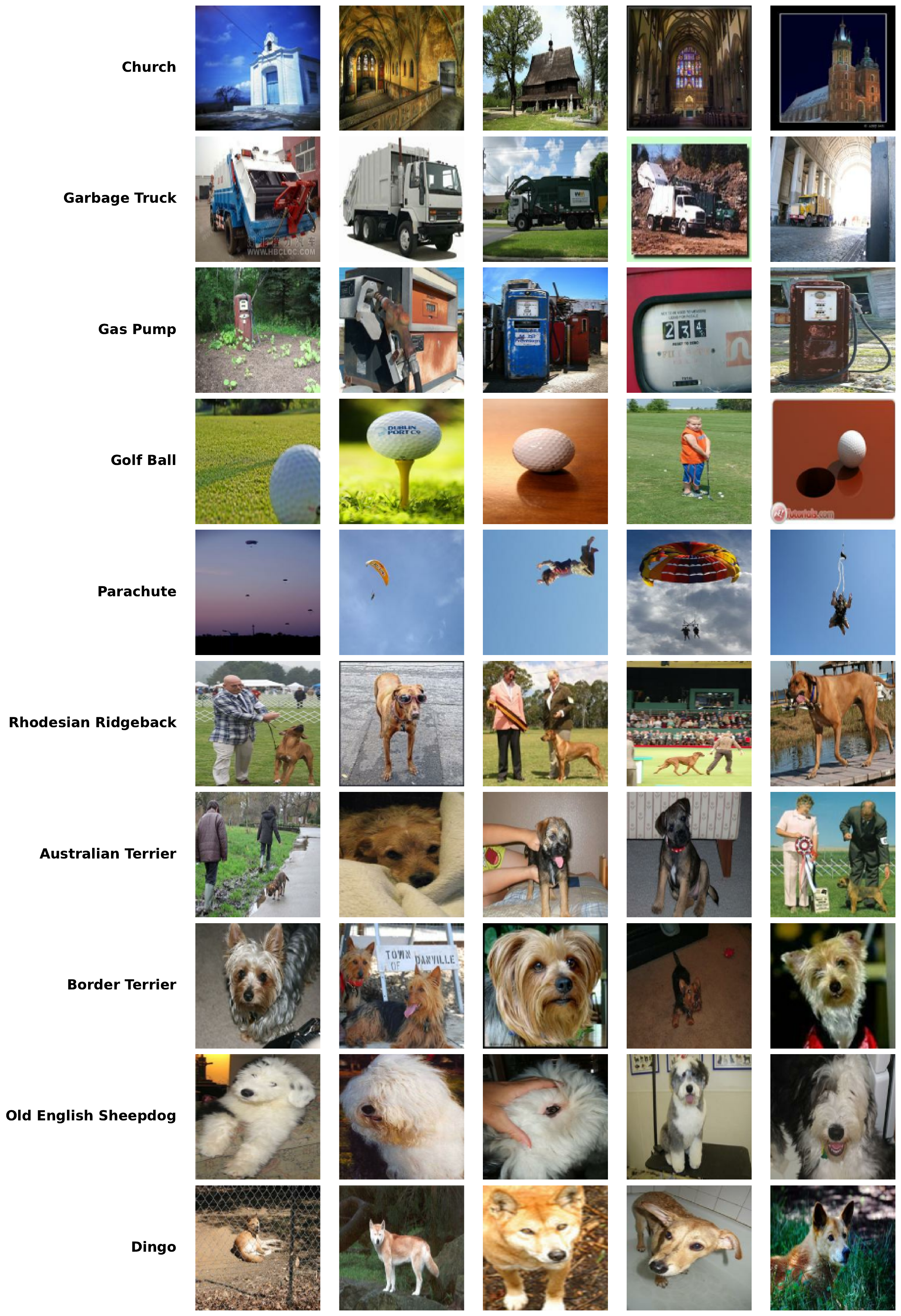}
    \caption{Sample images from the \textsf{WoofNette} dataset.}
    \label{fig:woofnette_data}
\end{figure*}

\end{document}